\documentclass[11pt]{article}
\usepackage{graphicx,amsmath,amsfonts,amssymb,bm,hyperref,url,breakurl,epsfig,epsf,color,fullpage,MnSymbol,mathbbol,fmtcount,algorithmic,algorithm,semtrans,cite,caption,subcaption,
  multirow}

\usepackage{caption}
\usepackage[bottom,hang,flushmargin]{footmisc}
\usepackage{comment}
\usepackage[bottom]{footmisc}
\usepackage{caption}
\usepackage{subcaption}
\usepackage{enumitem}

\newcommand{\bi}{\begin{itemize}}
\newcommand{\ei}{\end{itemize}}
\newcommand{\bal}{\begin{align}}
\newcommand{\eal}{\end{align}}

\newcommand{\EE}{\mathbb{E}}
\newcommand{\PP}{\mathbb{P}}

\newcommand{\bX}{\mathbf{X}}

\newcommand{\bZ}{\mathbf{Z}}
\newcommand{\bW}{\mathbf{W}}
\newcommand{\bx}{\mathbf{x}}
\newcommand{\by}{\mathbf{y}}
\newcommand{\bR}{\mathbf{R}}

\newcommand{\bw}{\mathbf{w}}

\newcommand{\bv}{\mathbf{v}}

\newcommand{\bu}{\mathbf{u}}
\newcommand{\bs}{\mathbf{s}}

\newcommand{\bGamma}{\mathbf{\Gamma}}

\newcommand{\bA}{\mathbf{A}}

\newcommand{\bC}{\mathbf{C}}
\newcommand{\bU}{\mathbf{U}}
\newcommand{\bV}{\mathbf{V}}
\newcommand{\bS}{\mathbf{S}}

\newcommand{\bI}{\mathbf{I}}

\newcommand{\bz}{\mathbf{z}}

\newcommand{\be}{\mathbf{e}}
\newcommand{\bH}{\mathbf{H}}
\newcommand{\bM}{\mathbf{M}}
\newcommand{\bq}{\mathbf{q}}
\newcommand{\bD}{\mathbf{D}}

\newcommand{\bzeta}{\mathbf{\zeta}}
\newcommand{\bK}{\mathbf{K}}
\newcommand{\bDelta}{\mathbf{\Delta}}
\newcommand{\tbDelta}{\tilde{\mathbf{\Delta}}}

\newcommand{\cN}{\mathcal{N}}

\newcommand{\cO}{\mathcal{O}}

\newcommand{\cG}{\mathcal{G}}
\newcommand{\cF}{\mathcal{F}}

\newcommand{\cL}{\mathcal{L}}
\newcommand{\cW}{\mathcal{W}}

\newcommand{\eps}{\epsilon}

\newcommand{\relu}{\text{relu}}
\newcommand{\rank}{{\rm{rank}}}
\newcommand{\bSigma}{\mathbf{\Sigma}}
\newcommand{\bsigma}{\mathbf{\sigma}}

\def\reals{\mathbb{R}}

\def\<{\langle}
\def\>{\rangle}

\usepackage{hyperref}
\definecolor{darkred}{RGB}{150,0,0}
\definecolor{darkgreen}{RGB}{0,150,0}
\definecolor{darkblue}{RGB}{0,0,200}
\hypersetup{colorlinks=true, linkcolor=darkred, citecolor=darkgreen, urlcolor=darkblue}

\numberwithin{equation}{section}

\def \endprf{\hfill {\vrule height6pt width6pt depth0pt}\medskip}
\newenvironment{proof}{\noindent {\bf Proof} }{\endprf\par}

\newtheorem{theorem}{\textbf{Theorem}}
\newtheorem{lemma}{\textbf{Lemma}}
\newtheorem{corollary}{\textbf{Corollary}}
\newtheorem{example}{\textbf{Example}}
\newtheorem{definition}{\textbf{Definition}}

 \newtheorem{conjecture}{\textbf{Conjecture}}
\newtheorem{proposition}{\textbf{Proposition}}
\newtheorem{assumption}{\textbf{Assumption}}
\newcommand{\diag}{{\rm{diag}}}

\title{{\huge Porcupine Neural Networks:\\ (Almost) All Local Optima are Global}}

\author{Soheil~Feizi, Hamid Javadi, Jesse Zhang and David Tse\\\\
Stanford University}
\date{}

\begin{document}
\maketitle

\begin{abstract}
Neural networks have been used prominently in several machine learning and statistics applications. In general, the underlying optimization of neural networks is non-convex which makes their performance analysis challenging. In this paper, we take a novel approach to this problem by asking whether one can constrain neural network weights to make its optimization landscape have good theoretical properties while at the same time, be a good approximation for the unconstrained one. For two-layer neural networks, we provide affirmative answers to these questions by introducing Porcupine Neural Networks (PNNs) whose weight vectors are constrained to lie over a finite set of lines. We show that most local optima of PNN optimizations are global while we have a characterization of regions where bad local optimizers may exist. Moreover, our theoretical and empirical results suggest that an unconstrained neural network can be approximated using a polynomially-large PNN.
\end{abstract}

\section{Introduction}
Neural networks have been used in several machine learning and statistical inference problems including regression and classification tasks. Some successful applications of neural networks and deep learning include speech recognition \cite{mohamed2012acoustic}, natural language processing \cite{collobert2008unified}, and image classiﬁcation \cite{krizhevsky2012imagenet}. The underlying neural network optimization is non-convex in general which makes its training NP-complete even for small networks \cite{blum1989training}. In practice, however, different variants of local search methods such as the gradient descent algorithm show excellent performance. Understanding the reason behind the success of such local search methods is still an open problem in the general case.

There has been several recent work in the theoretical literature aiming to study risk landscapes of neural networks and deep learning under various modeling assumptions. We review these work in Section \ref{subsec:prior}. In this paper, we study a key question whether an unconstrained neural network can be approximated with a constrained one whose optimization landscape has good theoretical properties. For two-layer neural networks, we provide an affirmative answer to this question by introducing a family of constrained neural networks which we refer to as {\it Porcupine Neural Networks (PNNs)} (Figure \ref{fig:PNN}). In PNNs, an incoming weight vector to a neuron is constrained to lie over a fixed line. For example, a neural network with multiple inputs and multiple neurons where each neuron is connected to one input is a PNN since input weight vectors to neurons lie over lines parallel to standard axes.

\begin{figure}[t]
\centering
  \includegraphics[width=0.7\linewidth]{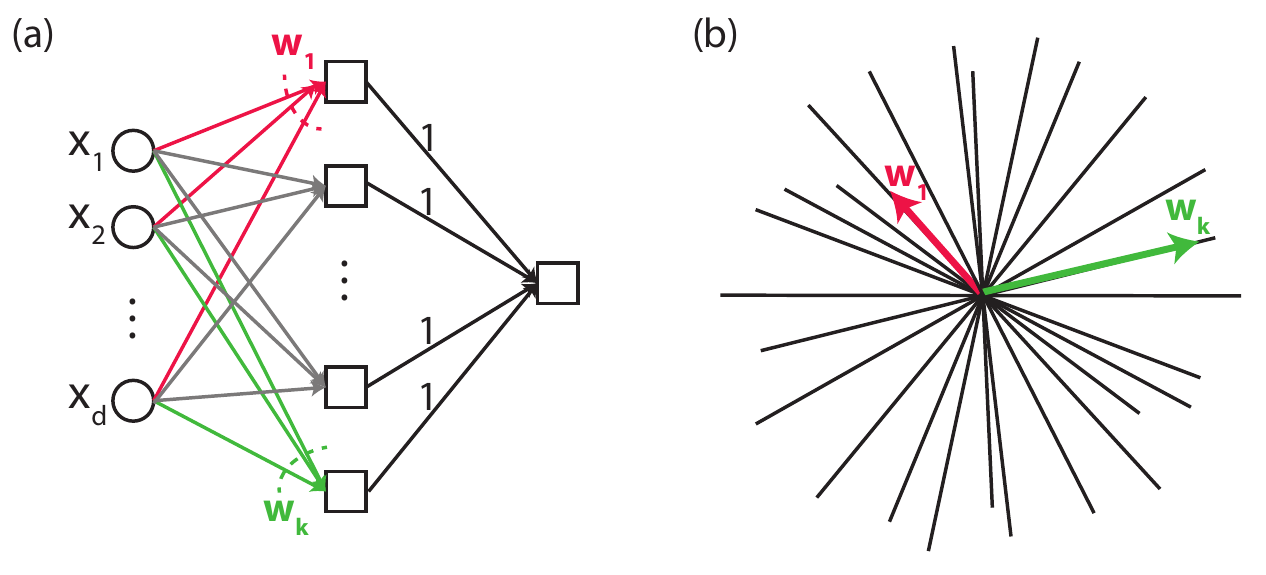}
\caption{(a) A two-layer Porcupine Neural Network (PNN). (b) In PNN, an incoming weight vector to a neuron is constrained to lie over a line in a $d$-dimensional space.}
\label{fig:PNN}
\end{figure}

We analyze population risk landscapes of two-layer PNNs with jointly Gaussian inputs and relu activation functions at hidden neurons. We show that under some modeling assumptions, most local optima of PNN optimizations are also global optimizers. Moreover, we characterize the parameter regions where bad local optima (i.e., local optimizers that are not global) may exist. In our analysis, we observe that a particular kernel function depicted in Figure \ref{fig:psi} plays an important role in characterizing population risk landscapes of PNNs. This kernel function is resulted from the computation of the covariance matrix of Gaussian variables restricted to a dual convex cone (Lemma \ref{lem:cov-truncated}). We will explain this observation in more detail.

Next, we study whether one can approximate an unconstrained (fully-connected) neural network function with a PNN whose number of neurons are polynomially-large in dimension. Our empirical results offer an affirmative answer to this question. For example, suppose the output data is generated using an unconstrained two-layer neural network with $d=15$ inputs and $k^*=20$ hidden neurons \footnote{Note that for both unconstrained and constrained neural networks, the second layer weights are assumed to be equal to one. The extension of the results to a more general case is an interesting direction for future work.}. Using this data, we train a random two-layer PNN with $k$ hidden neurons. We evaluate the PNN approximation error as the mean-squared error (MSE) normalized by the $L_2$ norm of the output samples in a two-fold cross validation setup. As depicted in Figure \ref{fig:PNN-intro}, by increasing the number of neurons of PNN, the PNN approximation error decreases. Notably, to obtain a relatively small approximation error, PNN's number of hidden neurons does not need to be exponentially large in dimension. We explain details of this experiment in Section \ref{sec:exp}.

\begin{figure}[t]
\centering
  \includegraphics[width=0.6\linewidth]{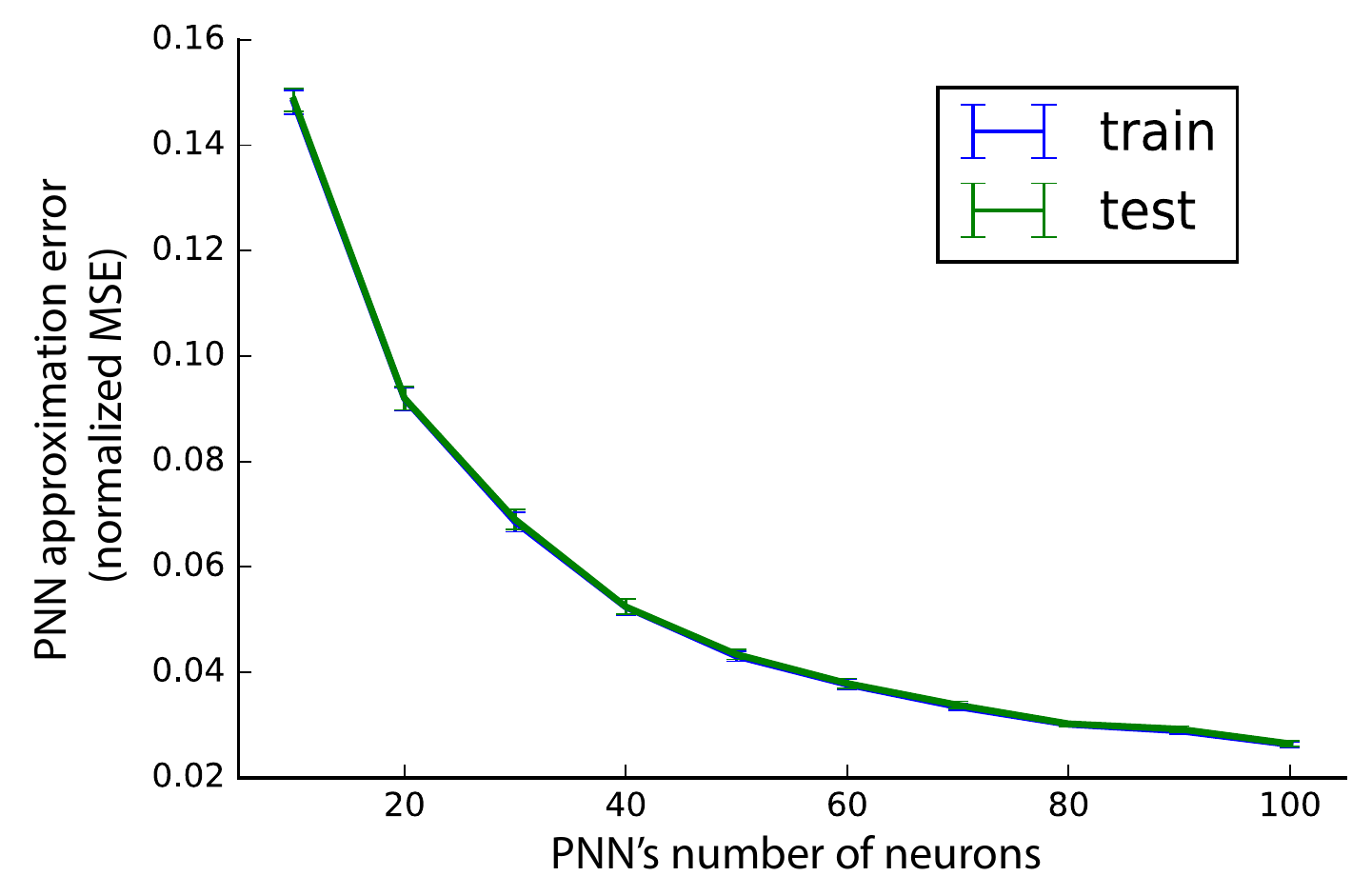}
\caption{Approximations of an unconstrained two-layer neural network with $d=15$ inputs and $k^*=20$ hidden neurons using random two-layer PNNs.}
\label{fig:PNN-intro}
\end{figure}

In Section \ref{sec:apx}, we study a characterization of the PNN approximation error with respect to the input dimension and the complexity of the unconstrained neural network function. We show that under some modeling assumptions, the PNN approximation error can be bounded by the spectral norm of the generalized Schur complement of a kernel matrix. We analyze this bound for random PNNs in the high-dimensional regime when the ground-truth data is generated using an unconstrained neural network with random weights. For the case where the dimension of inputs and the number of hidden neurons increase with the same rate, we compute the asymptotic limit. Moreover, we provide numerical results for the case when the number of hidden neurons grows with a polynomial rate in dimension. We also analyze a naive minimax approximation bound which requires PNN's number of neurons to be exponentially large in dimension.

Finally, in Section \ref{sec:conc}, we discuss how the proposed PNN framework can potentially be used to {\it explain} the success of local search methods such as gradient descent in solving the unconstrained neural network optimization.

\subsection{Notation}\label{sec:notation}
For matrices we use bold-faced upper case letters, for vectors we use bold-faced lower case letters, and for scalars we use regular lower case letters. For example, $\bX$ represents a matrix, $\bx$ represents a vector, and $x$ represents a scalar number. $\bI_{n}$ is the identity matrix of size $n\times n$. $\be_j$ is a vector whose $j$-th element is non-zero and its other elements are zero. $\mathbf{1}_{n_1,n_2}$ is the all one matrix of size $n_1\times n_2$. When no confusion arises, we drop the subscripts. $\mathbf{1}\{x=y\}$ is the indicator function which is equal to one if $x=y$, otherwise it is zero. $\relu(x)=\max(x,0)$. $Tr(\bX)$ and $\bX^t$ represent the trace and the transpose of the matrix $\bX$, respectively. $\|\bx\|_{2}=\bx^t \bx$ is the second norm of the vector $\bx$. When no confusion arises, we drop the subscript. $\|\bx\|_{1}$ is the $l_1$ norm of the vector $\bx$. $\|\bX\|$ is the operator (spectral) norm of the matrix $\bX$. $\|\bx\|_{0}$ is the number of non-zero elements of the vector $\bx$. $<\bx,\by>$ is the inner product between vectors $\bx$ and $\by$. $\bx \perp \by$ indicates that vectors $\bx$ and $\by$ are orthogonal. $\theta_{\bx,\by}$ is the angle between vectors $\bx$ and $\by$. $\cN(\mu,\bGamma)$ is the Gaussian distribution with mean $\mu$ and the covariance $\bGamma$. $f[\bA]$ is a matrix where the function $f(.)$ is applied to its components, i.e., $f[\bA](i,j)=f(\bA(i,j))$. $\bA^{\dagger}$ is the pseudo inverse of the matrix $\bA$. The eigen decomposition of the matrix $\bA\in\mathbb{R}^{n\times n}$ is denoted by $\bA=\sum_{i=1}^{n} \lambda_i(\bA) \bu_i(\bA) \bu_i(\bA)^t$, where $\lambda_i(\bA)$ is the $i$-th largest eigenvalue of the matrix $\bA$ corresponding to the eigenvector $\bu_i(\bA)$. We have $\lambda_1(\bA)\geq \lambda_2(\bA)\geq \cdots$.


\section{Unconstrained Neural Networks}\label{sec:formulation}
Consider a two-layer neural network with $k$ neurons where the input is in $\mathbb{R}^d$ (Figure \ref{fig:PNN}-a). The weight vector from the input to the $i$-th neuron is denoted by $\bw_i\in \mathbb{R}^d$. For simplicity, we assume that second layer weights are equal to one another. Let
\begin{align}\label{eq:output}
h(\bx; \bW):= \sum_{i=1}^{k} \phi\left(\bw_i^t \bx\right),
\end{align}
where $\bx=(x_1,...,x_d)^t$ and $\bW:=(\bw_1,\bw_2,...,\bw_k)\in\cW\subseteq \reals^{d\times k}$. The activation function at each neuron is assumed to be $\phi(z):=\relu(z)=\max(z,0)$.

Consider $\mathcal{F}$, the set of all functions
$f: \reals^d \to \reals$ where $f$ can be realized with a neural network described in \eqref{eq:output}. In other words,
\begin{align}\label{eq:Fdef}
\mathcal F := \left\{ f: \reals^d \to \reals;\;\; \exists \bW \in \cW, \; f(\bx) = h(\bx; \bW),\; \forall \bx \in \reals^d\right\}.
\end{align}
In a fully connected neural network structure, $\cW=\mathbb{R}^{d\times k}$. We refer to this case as the unconstrained neural network. Note that particular network architectures can impose constraints on $\cW$.

Let $\bx\sim \cN(0,\bI)$. We consider the population risk defined as the mean squared error (MSE):
\begin{align}\label{eq:loss-function}
L(\bW):= \EE\left[\left(h(\bx; \bW)-y\right)^2\right],
\end{align}
where $y$ is the output variable. If $y$ is generated by a neural network with the same architecture as of \eqref{eq:output}, we have $y=h(\bx; \bW_{true})$.

Understanding the population risk function is an important step towards characterizing the empirical risk landscape \cite{mei2016landscape}. In this paper, for simplicity, we only focus on the population risk.

The neural network optimization can be written as follows:
\begin{align}\label{opt:neural-network}
\min_{\bW}\quad &L(\bW)\\
&\bW\in \cW.\nonumber
\end{align}

Let $\bW^*$ be a global optimum of this optimization. $L(\bW^*)=0$ means that $y$ can be generated by a neural network with the same architecture (i.e., $\bW_{true}$ is a global optimum.). We refer to this case as the {\it matched} neural network optimization. Moreover, we refer to the case of $L(\bW^*)>0$ as the {\it mismatched} neural network optimization. Optimization \eqref{opt:neural-network} in general is non-convex owing to nonlinear activation functions in neurons.

\section{Porcupine Neural Networks}\label{sec:PNN}
Characterizing the landscape of the objective function of optimization \eqref{opt:neural-network} is challenging in general. In this paper, we consider a constrained version of this optimization where weight vectors belong to a finite set of lines in a $d$-dimensional space (Figure \ref{fig:PNN}). This constraint may arise either from the neural network architecture or can be imposed by design.

Mathematically, let $\cL=\{L_1,...,L_r\}$ be a set of lines in a $d$-dimensional space. Let $\cG_{i}$ be the set of neurons whose incoming weight vectors lie over the line $L_i$. Therefore, we have $\cG_1\cup ... \cup \cG_r =\{1,...,k\}$. Moreover, we assume $\cG_i\neq \emptyset$ for $1\leq i\leq r$ otherwise that line can be removed from the set $\cL$. For every $j\in \cG_i$, we define the function $g(.)$ such that $g(j)=i$.

For a given set $\cL$ and a neuron-to-line mapping $\cG$, we define $\cF_{\cL,\cG}\subseteq \cF$ as the set of all functions that can be realized with a neural network \eqref{eq:output} where $\bw_i$ lies over the line $L_{g(i)}$. Namely,

\begin{align}\label{eq:PNN-fun}
\cF_{\cL,\cG} := \left\{ f: \reals^d \to \reals;\;\; \exists \bW=(\bw_1,...,\bw_k), \bw_i\in L_{g(i)}, \; f(\bx) = h(\bx; \bW),\; \forall \bx \in \reals^d\right\}.
\end{align}
We refer to this family of neural networks as Porcupine Neural Networks (PNNs).

\begin{figure}
\centering
  \includegraphics[width=0.7\linewidth]{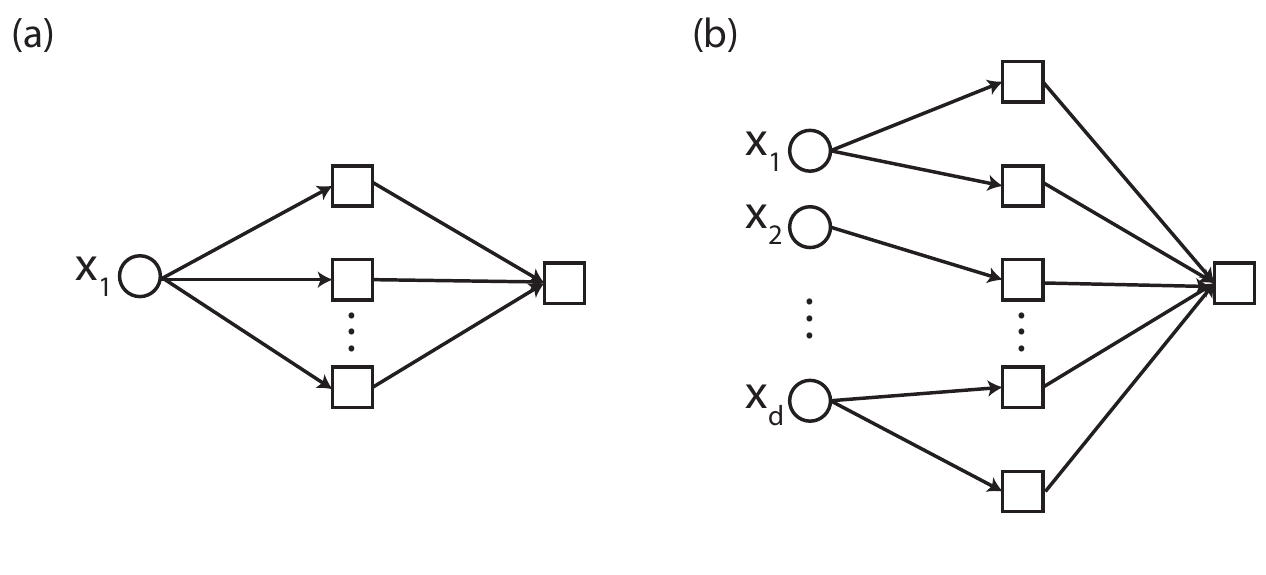}
\caption{Examples of (a) scalar PNN, and (b) degree-one PNN structures.}
\label{fig:PNN-ex}
\end{figure}

In some cases, the PNN constraint is imposed by the neural network architecture. For example, consider the neural network depicted in Figure \ref{fig:PNN-ex}-a, which has a single input and $k$ neurons. In this network structure, $\bw_i$'s are scalars. Thus, every realizable function with this neural network can be realized using a PNN where $\cL$ includes a single line. We refer to this family of neural networks as scalar PNNs. Another example of porcupine neural networks is depicted in Figure \ref{fig:PNN-ex}-b. In this case, the neural network has multiple inputs and multiple neurons. Each neuron in this network is connected to one input. Every realizable function with this neural network can be described using a PNN whose lines are parallel to standard axes. We refer to this family of neural networks as degree-one PNNs. Scalar PNNs are also degree-one PNNs. However, since their analysis is simpler, we make such a distinction.

In general, functions described by PNNs (i.e., $\cF_{\cL,\cG}$) can be viewed as angular discretizations of functions described by unconstrained neural networks (i.e., $\cF$). By increasing the size of $|\cL|$ (i.e., the number of lines), we can approximate every $f\in \cF$ by $\hat{f}\in \cF_{\cL,\cG}$ arbitrarily closely. Thus, characterizing the landscape of the loss function over PNNs can help us to understand the landscape of the unconstrained loss function.

The PNN optimization can be written as
\begin{align}\label{opt:neural-network-PNN}
\min_{\bW}\quad &L(\bW)\\
&\bw_i\in L_{g(i)}\quad 1\leq i\leq k.\nonumber
\end{align}
Matched and mismatched PNN optimizations are defined similar to the unconstrained ones. In this paper, we characterize the population risk landscape of the PNN optimization \eqref{opt:neural-network-PNN} in both matched and mismatched cases. In Section \ref{sec:matched}, we consider the matched PNN optimization, while in Section \ref{sec:mismatched}, we study the mismatched one. Then, in Section \ref{sec:apx}, we study approximations of unconstrained neural network functions with PNNs.

Note that a PNN can be viewed as a neural network whose feature vectors (i.e., input weight vectors to neurons) are fixed up to scalings due to the PNN optimization. This view can relate a random PNN (i.e., a PNN whose lines are random) to the application of random features in kernel machines \cite{rahimi2008random}. Although our results in Sections \ref{sec:matched}, \ref{sec:mismatched} and \ref{sec:apx} are for general PNNs, we study them for random PNNs in Section \ref{sec:apx} as well.

\section{Related Work}\label{subsec:prior}
To explain the success of neural networks, some references study their ability to approximate smooth functions \cite{barron1993universal,yukich1995sup,klusowski2016uniform,klusowski2017minimax,lee2017ability,ferrari2005smooth,giryes2016deep}, while some other references focus on benefits of having more layers \cite{telgarsky2016benefits,liang2016deep}. Over-parameterized networks where the number of parameters are larger than the number of training samples have been studied in \cite{soltanolkotabi2017theoretical,nguyen2017loss}. However, such architectures can cause generalization issues in practice \cite{zhang2016understanding}.

References \cite{mei2016landscape,hazan2015beyond,kakade2011efficient,soltanolkotabi2017learning} have studied the convergence of the local search algorithms such as gradient descent methods to the global optimum of the neural network optimization with zero hidden neurons and a single output. In this case, the loss function of the neural network optimization has a single local optimizer which is the same as the global optimum. However, for neural networks with hidden neurons, the landscape of the loss function is more complicated than the case with no hidden neurons.

Several work has studied the risk landscape of neural network optimizations for more complex structures under various model assumptions \cite{kawaguchi2016deep,yun2017global,soudry2016no,choromanska2015loss,tian2016symmetry,tian2017analytical,zhong2017recovery,brutzkus2017globally,zhang2017electron,li2017convergence,janzamin2015beating}. Reference \cite{kawaguchi2016deep} shows that in the linear neural network optimization, the population risk landscape does not have any bad local optima. Reference \cite{yun2017global} extends these results and provides necessary and sufficient conditions for a critical point of the loss function to be a global minimum. Reference \cite{soudry2016no} shows that for a two-layer neural network with leaky activation functions, the gradient descent method on a modified loss function converges to a global optimizer of the modified loss function which can be different from the original global optimum. Under an independent activations assumption, reference \cite{choromanska2015loss} simplifies the loss function of a neural network optimization to a polynomial and shows that local optimizers obtain approximately the same objective values as the global ones. This result has been extended by reference \cite{kawaguchi2016deep} to show that all local minima are global minima in a nonlinear network. However, the underlying assumption of having independent activations at neurons usually are not satisfied in practice.

References \cite{tian2016symmetry,tian2017analytical,zhong2017recovery} consider a two-layer neural network with Gaussian inputs under a matched (realizable) model where the output is generated from a network with planted weights. Moreover, they assume the number of neurons in the hidden layer is smaller than the dimension of inputs. This critical assumption makes the loss function positive-definite in a small neighborhood near the global optimum. Then, reference \cite{zhong2017recovery} provides a tensor-based method to initialize the local search algorithm in that neighborhood which guarantees its convergence to the global optimum. In our problem formulation, the number of hidden neurons can be larger than the dimension of inputs as it is often the case in practice. Moreover, we characterize risk landscapes for a certain family of neural networks in all parameter regions, not just around the global optimizer. This can guide us towards understanding the reason behind the success of local search methods in practice.

For a neural network with a single non-overlapping convolutional layer, reference \cite{brutzkus2017globally} shows that all local optimizers of the loss function are global optimizers as well. They also show that in the overlapping case, the problem is NP-hard when inputs are not Gaussian. Moreover, reference \cite{zhang2017electron} studies this problem with non-standard activation functions, while reference \cite{li2017convergence} considers the case where the weights from the hidden layer to the output are close to the identity. Other related works include improper learning models using kernel based approaches \cite{goel2016reliably,zhang2016l1} and a method of moments estimator using tensor decomposition \cite{janzamin2015beating}.

\section{Population Risk Landscapes of Matched PNNs}\label{sec:matched}
In this section, we analyze the population risk landscape of matched PNNs. In the matched case, the set of lines $\cL$ and the neuron-to-line mapping $\cG$ of a PNN used for generating the data are assumed to be known in training as well. We consider the case where these are unknowns in training in Section \ref{sec:mismatched}.

\begin{figure}
\centering
  \includegraphics[width=0.4\linewidth]{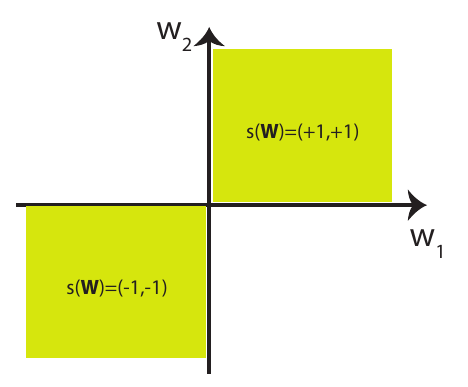}
\caption{For the scalar PNN, parameter regions where $s(\bW)=\pm \mathbf{1}$ may include bad local optima. In other regions, all local optima are global. This figure highlights regions where $s(\bW)=\pm \mathbf{1}$ for a scalar PNN with two neurons.}
\label{fig:region-ex}
\end{figure}

\subsection{Scalar PNNs}\label{sec:scalar}
In this section, we consider a neural network structure with a single input and multiple neurons (i.e., $d=1$, $k>1$). Such neural networks are PNNs with $\cL$ containing a single line. Thus, we refer to them as scalar PNNs. An example of a scalar PNN is depicted in Figure \ref{fig:PNN-ex}-a. In this case, every $\bw_i$ for $1\leq i\leq k$ is a single scalar. We refer to that element by $w_i$. We assume $w_i$'s are non-zero, otherwise the neural network structure can be reduced to another structure with fewer neurons.

\begin{theorem}\label{thm:scalar-global}
The loss function \eqref{eq:loss-function} for a scalar PNN can be written as
\begin{align}\label{eq:loss-scalar}
L(\bW)=\frac{1}{4}\left(\sum_{i=1}^{k} w_i-\sum_{i=1}^{k} w_i^*\right)^2+\frac{1}{4}\left(\sum_{i=1}^{k} |w_i|-\sum_{i=1}^{k} |w_i^*|\right)^2.
\end{align}
\end{theorem}

\begin{proof}
See Section \ref{subsec:proof-thm-global}.
\end{proof}

Since for a scalar PNN, the loss function $L(\bW)$ can be written as sum of squared terms, we have the following corollary:

\begin{corollary}\label{cor:scalar-global}
For a scalar PNN, $\bW$ is the global optimizer of optimization \eqref{opt:neural-network-PNN} if and only if
\begin{align}\label{eq:global}
 \sum_{i=1}^{k} w_i&= \sum_{i=1}^{k} w_i^*,\\
 \sum_{i=1}^{k} |w_i|&= \sum_{i=1}^{k} |w_i^*|.\nonumber
\end{align}
\end{corollary}
Next, we characterize local optimizers of optimization \eqref{opt:neural-network-PNN}.

Let $s(w_i)$ be the sign variable of $w_i$, i.e., $s(w_i)=1$ if $w_i>0$, otherwise $s(w_i)=-1$. Let $s(\bW)\triangleq (s(w_1),...,s(w_k))^t$. Let $R(\bs)$ denote the space of all $\bW$ where $s_i=s(w_i)$, i.e., $R(\bs)\triangleq \{(w_1,...,w_k): s(w_i)=s_i\}$.


\begin{theorem}\label{thm:local}
If $s(\bW^*)\neq \pm \mathbf{1}$:
\begin{itemize}
  \item [-] In every region $R(\bs)$ whose $\bs \neq \pm \mathbf{1}$, optimization \eqref{opt:neural-network-PNN} only has global optimizers without any bad local optimizers.
  \item [-] In two regions $R(\mathbf{1})$ and $R(-\mathbf{1})$, optimization \eqref{opt:neural-network-PNN} does not have global optimizers and only has bad local optimizers.
\end{itemize}
If $s(\bW^*)= \pm \mathbf{1}$:
\begin{itemize}
  \item [-] In regions $R(\bs)$ where $\bs \neq \pm \mathbf{1}$ and in the region $R(-s(\bW^*))$, optimization \eqref{opt:neural-network-PNN} neither has global nor bad local optimizers.
  \item [-] In the region $R(s(\bW^*))$, optimization \eqref{opt:neural-network-PNN} only has global optimizers without any bad local optimizers.
\end{itemize}
\end{theorem}

\begin{proof}
See Section \ref{subsec:proof-thm-local}.
\end{proof}

Theorem \ref{thm:local} indicates that optimization \eqref{opt:neural-network-PNN} can have bad local optimizers. However, this can occur only in two parameter regions, out of $2^k$ regions, which can be checked separately (Figure \ref{fig:region-ex}). Thus, a variant of the gradient descent method which checks these cases separately converges to a global optimizer.

\begin{figure}
\centering
  \includegraphics[width=\linewidth]{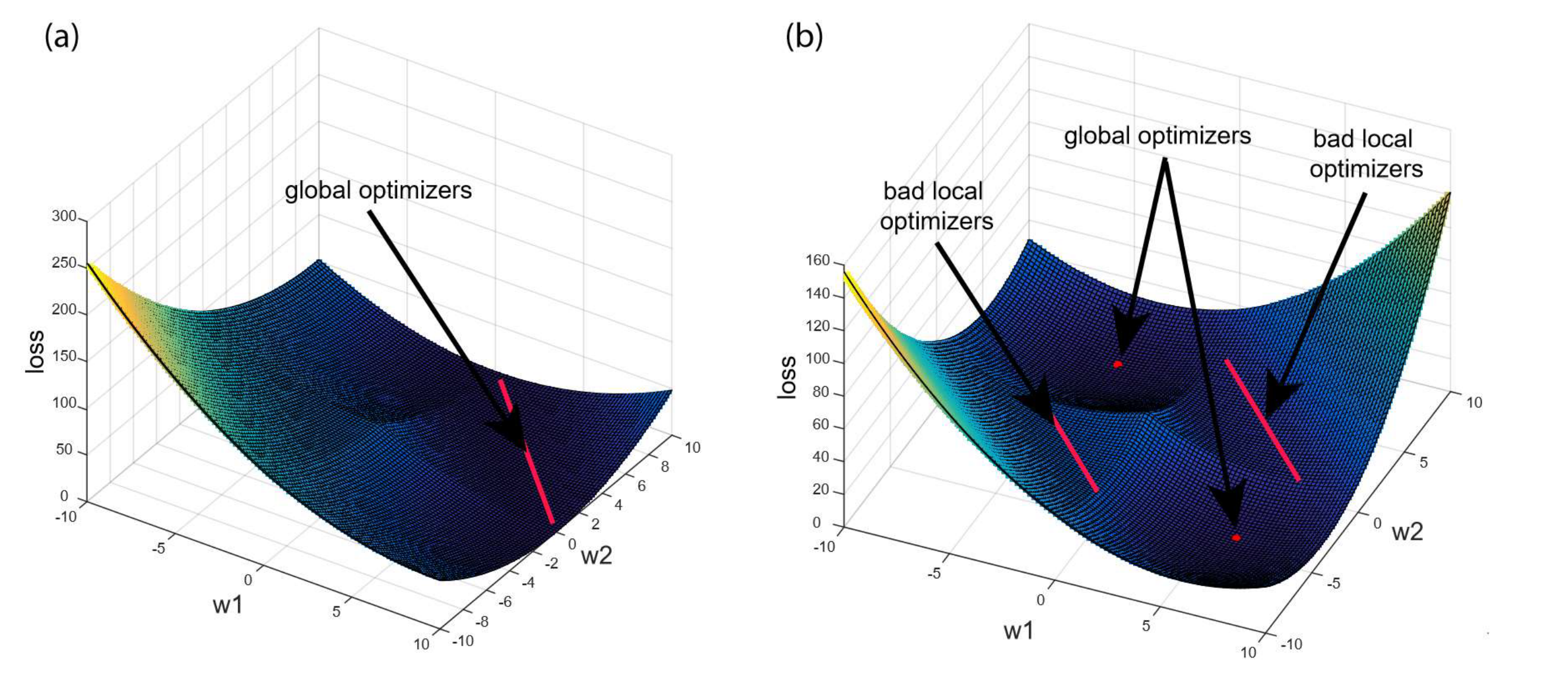}
\caption{The landscape of the loss function for a scalar PNN with two neurons. In panel (a), we consider $w_1^*=6$ and $w_2^*=4$, while in panel (b), we have $w_1^*=6$ and $w_2^*=-4$. According to Theorem \ref{thm:local}, in the case of panel (a), the loss function does not have bad local optimizers, while in the case of panel (b), it has bad local optimizers in regions $R\left((-1,-1)\right)$ and $R\left((1,1)\right)$.}
\label{fig:scalar}
\end{figure}

Next, we characterize the Hessian of the loss function:
\begin{theorem}\label{thm:hessian-scalar}
For a scalar PNN, in every region $R(\bs)$, the Hessian matrix of the loss function $L(\bW)$ is positive semidefinite, i.e.,  in every region $R(\bs)$, the loss function is convex. In regions $R(\bs)$ where $\bs\neq \pm \mathbf{1}$, the rank of the Hessian matrix is two, while in two regions $R(\pm \mathbf{1})$, the rank of the Hessian matrix is equal to one.
\end{theorem}

\begin{proof}
See Section \ref{subsection:proof-hessian-scalar}.
\end{proof}

Finally, for a scalar PNN, we illustrate the landscape of the loss function with an example. Figure \ref{fig:scalar} considers the case with a single input and two neurons (i.e., $d=1$, $k=2$). In Figure \ref{fig:scalar}-a, we assume $w_1^*=6$ and $w_2^*=4$. According to Theorem \ref{thm:local}, only the region $R\left((1,1)\right)$ contains global optimizers (all points in this region on the line $w_1+w_2=10$ are global optimizers.). In Figure \ref{fig:scalar}-b, we consider $w_1^*=6$ and $w_2^*=-4$. According to Theorem \ref{thm:local}, regions $R\left((1,-1)\right)$ and $R\left((-1,1)\right)$ have global optimizers, while regions $R\left((1,1)\right)$ and $R\left((-1,-1)\right)$ include bad local optimizers.

\subsection{Degree-One PNNs}\label{sec:degree-one}
In this section, we consider a neural network structure with more than one input and multiple neurons ($d\geq 1$ and $k\geq 1$) such that each neuron is connected to one input. Such neural networks are PNNs whose lines are parallel to standard axes. Thus, we refer to them as degree-one PNNs.

Similar to the scalar PNN case, in the case of the degree-one PNN, every $\bw_i$ has one non-zero element. We refer to that element by $w_i$. Let $\cG_{r}$ be the set of neurons that are connected to the variable $x_r$, i.e., $\cG_{r}=\{j:\bw_j(r)\neq 0\}$. Therefore, we have $\cG_1\cup ... \cup \cG_d =\{1,...,k\}$. Moreover, we assume $\cG_i\neq \emptyset$ for $1\leq i\leq d$, i.e., there is at least one neuron connected to each input variable. For every $j\in \cG_r$, we define the function $g(.)$ such that $g(j)=r$ \footnote{These definitions match with definitions of $\cG$ and $g(.)$ for a general PNN.}. Moreover, we define
\begin{align}\label{eq:q}
 q_r&:=\sum_{i\in \cG_r} \|\bw_i\|, \\
 q_r^*&:=\sum_{i\in \cG_r} \|\bw_i^*\|. \nonumber
\end{align}
Finally, we define $\bq :=(q_1,...,q_d)^t$ and $\bq^* :=(q_1^*,...,q_d^*)^t$.

\begin{theorem}\label{thm:degree-one-global}
The loss function \eqref{eq:loss-function} for a degree-one PNN can be written as
\begin{align}\label{eq:loss-degree-one}
L(\bW)=\frac{1}{4}\|\sum_{i=1}^{k} \bw_i-\sum_{i=1}^{k}\bw_i^*\|^2+\frac{1}{4}(\bq-\bq^*)^t \bC (\bq-\bq^*),
\end{align}
where
\begin{align}\label{eq:C1}
\bC=\left( {\begin{array}{cccc}
   1 & \frac{2}{\pi} & \cdots & \frac{2}{\pi} \\
   \frac{2}{\pi} & 1 & \cdots & \frac{2}{\pi} \\
  \vdots &  & \ddots & \vdots \\
  \frac{2}{\pi} & &\cdots & 1
  \end{array} } \right).
\end{align}
\end{theorem}

\begin{proof}
See Section \ref{subsec:proof-degree-one-global}.
\end{proof}

Since $\bC$ is a positive definite matrix, we have the following corollary:
\begin{corollary}\label{cor:degree-one-global}
$\bW^*$ is a global optimizer of optimization \eqref{opt:neural-network-PNN} for a degree-one PNN if and only if
\begin{align}\label{eq:degree-one-global-cond}
\sum_{i\in \cG_r} w_i&= \sum_{i\in \cG_r} w_i^*,\quad 1\leq r\leq d\\
q_i&= q_i^*,\quad 1\leq r\leq d.\nonumber
\end{align}
\end{corollary}
Next, we characterize local optimizers of optimization \eqref{opt:neural-network-PNN} for degree-one PNNs. The sign variable assigned to the weight vector $\bw_j$ is defined as the sign of its non-zero element, i.e., $s(\bw_j)=s(w_j)$ where $w_j$ is the non-zero element of $\bw_j$. Define $R(\bs_1,...,\bs_d)$ as the space of $\bW$ where $\bs_i$ is the sign vector of weights $\bw_j$ connected to input $x_i$ (i.e., $j\in \cG_i$).
\begin{theorem}\label{thm:one-degree-local}
For a degree-one PNN, in regions $R(\bs_1,...,\bs_d)$ where $\bs_i\neq \pm \mathbf{1}$ for $1\leq i\leq d$, every local optimizer is a global optimizer. In other regions, we may have bad local optima.
\end{theorem}

\begin{proof}
See Section \ref{subsec:proof-degree-one-local}.
\end{proof}

In practice, if the gradient descent algorithm converges to a point in a region $R(\bs_1,...,\bs_d)$ where signs of weight vectors connected to an input are all ones or minus ones, that point may be a bad local optimizer. Thus, one may re-initialize the gradient descent algorithm in such cases. We show this effect through simulations in Section \ref{sec:exp}.

\subsection{General PNNs}\label{sec:quantized}
In this section, we characterize the landscape of the loss function for a general PNN. Recall that $\cL=\{L_1,...,L_r\}$ is the set of lines in a $d$-dimensional space. Vectors over a line $L_i$ can have two orientations. We say a vector has a positive orientation if its component in the largest non-zero index is positive. Otherwise, it has a negative orientation. For example, $\bw_1=(-1,2,0,3,0)$ has a positive orientation because $\bw_1(4)>0$, while the vector $\bw_2=(-1,2,0,0,-3)$ has a negative orientation because $\bw_2(5)<0$. Mathematically, let $\mu(\bw_i)$ be the largest index of the vector $\bw_i$ with a non-zero entry, i.e., $\mu(\bw_i)=\arg\max_{j} (\bw_i(j)\neq 0)$. Then, $s(\bw_i)=1$ if $\mu(\bw_i)>0$, otherwise $s(\bw_i)=-1$.

Let $\bu_i$ be a unit norm vector on the line $L_i$ such that $s(\bu_i)=1$. Let $\bU_{\cL}=(\bu_1,...,\bu_r)$. Let $\bA_{\cL}\in\mathbb{R}^{r\times r}$ be a matrix such that its $(i,j)$-component is the angle between lines $L_i$ and $L_j$, i.e., $\bA_{\cL}(i,j)=\theta_{\bu_i,\bu_{j}}$. Moreover, let $\bK_{\cL}=\bU_{\cL}^t \bU_{\cL}=\cos[\bA_{\cL}]$.

Recall that $\cG_i$ is the set of neurons whose incoming weight vectors lie over the line $L_i$, i.e., $\cG_i\triangleq \{j:\bw_j\in L_i\}$. Moreover, if $j\in \cG_i$, we define $g(j)=i$. In the degree-one PNN explained in Section \ref{sec:degree-one}, each line corresponds to an input because $\cL$ contains lines parallel to standard axes. However, for a general PNN, we may not have such a correspondence between lines and inputs.

With these notations, for $\bw_j\in L_i$, we have
\begin{align}\label{eq:quan-wi}
\bw_j=\|\bw_j\|s(\bw_j)\bu_{g(j)}.
\end{align}
Moreover, for every $\bw_i$ and $\bw_j$, we have
\begin{align}\label{eq:quan-angle-wi-wj}
\theta_{\bw_i,\bw_j}=\frac{\pi}{2}+(a_{g(i),g(j)}-\frac{\pi}{2})s(\bw_i)s(\bw_j).
\end{align}
Define the kernel function $\psi:[-1,1]\to \mathbb{R}$ as
\begin{align}\label{eq:psi-def}
\psi(x)=x+\frac{2}{\pi}\left(\sqrt{1-x^2}-x\cos^{-1}(x)\right).
\end{align}

\begin{figure}[t]
\centering
  \includegraphics[width=0.4\linewidth]{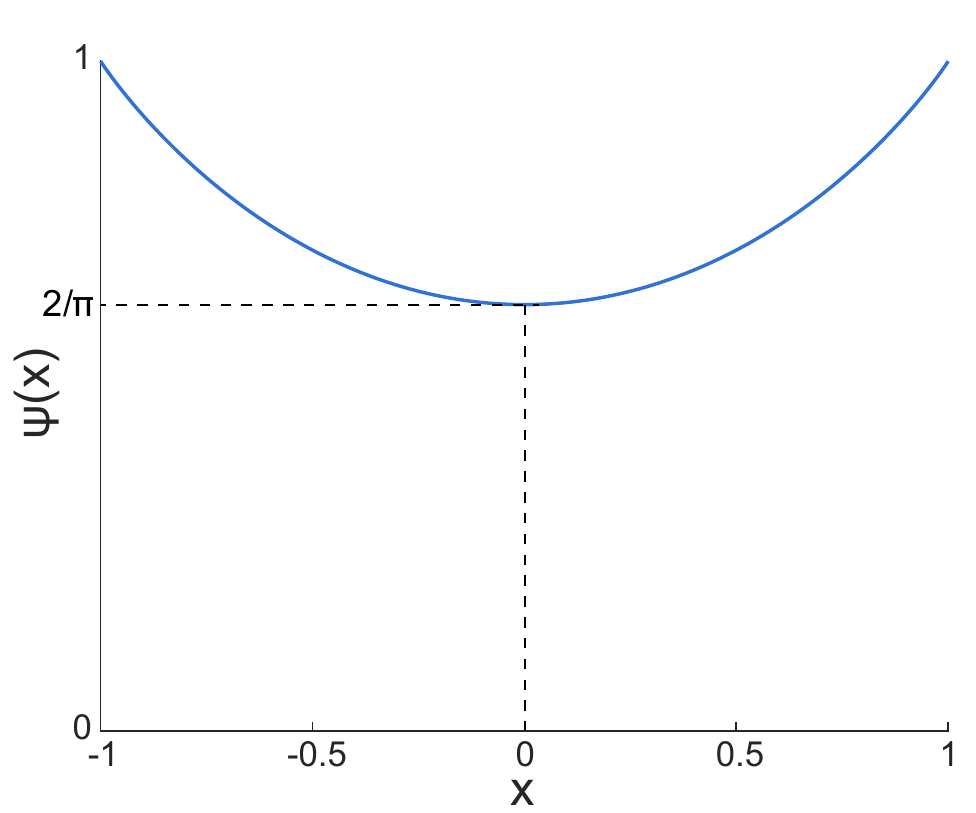}
\caption{An illustration of the kernel function $\psi(x)$ defined as in \eqref{eq:psi-def}.}
\label{fig:psi}
\end{figure}

In the following Theorem, we show that this kernel function, which is depicted in Figure \ref{fig:psi}, plays an important role in characterizing optimizers of optimization \eqref{opt:neural-network-PNN}. In particular, we show that the objective function of the neural network optimization has a term where this kernel function is applied (component-wise) to the inner product matrix among vectors $\bu_1$,...,$\bu_r$.

\begin{theorem}\label{thm:global-quan}
The loss function \eqref{eq:loss-function} for a matched PNN can be written as
\begin{align}\label{eq:loss-quan}
L(\bW)=\frac{1}{4}\|\sum_{i=1}^{k} \bw_i-\bw_i^*\|^2+\frac{1}{4}(\bq-\bq^*)^t \psi[\bK_{\cL}] (\bq-\bq^*),
\end{align}
where $\psi(.)$ is defined as in \eqref{eq:psi-def} and $\bq$ and $\bq^*$ are defined as in \eqref{eq:q}.
\end{theorem}
\begin{proof}
See Section \ref{subsec:proof:qlobal-quan}.
\end{proof}
For the degree-one PNN where $\bu_i=\be_i$ for $1\leq i\leq d$, the matrix $\bC$ of \eqref{eq:C1} and the matrix $\psi[\bK_{\cL}]$ are the same.

The kernel function $\psi(.)$ has a linear term and a nonlinear term. Note that the inner product matrix $\bK_{\cL}$ is positive semidefinite. Below, we show that applying the kernel function $\psi(.)$ (component-wise) to $\bK_{\cL}$ preserves this property.

\begin{lemma}\label{lem:C-psd}
For every $\cL$, $\psi[\bK_{\cL}]$ is positive semidefinite.
\end{lemma}
\begin{proof}
See Section \ref{subsec:proof-lem-C-PSD}.
\end{proof}

\begin{corollary}\label{cor:general-global}
If $\psi[\bK_{\cL}]$ is a positive definite matrix, $\bW^*$ is a global optimizer of optimization \eqref{opt:neural-network-PNN} if and only if
\begin{align}\label{eq:quan-global-cond}
\sum_{i=1}^{k} \bw_i&= \sum_{i=1}^{k} \bw_i^*,\\
\bq&= \bq^*.\nonumber
\end{align}
\end{corollary}

\begin{figure}[t]
\centering
  \includegraphics[width=0.9\linewidth]{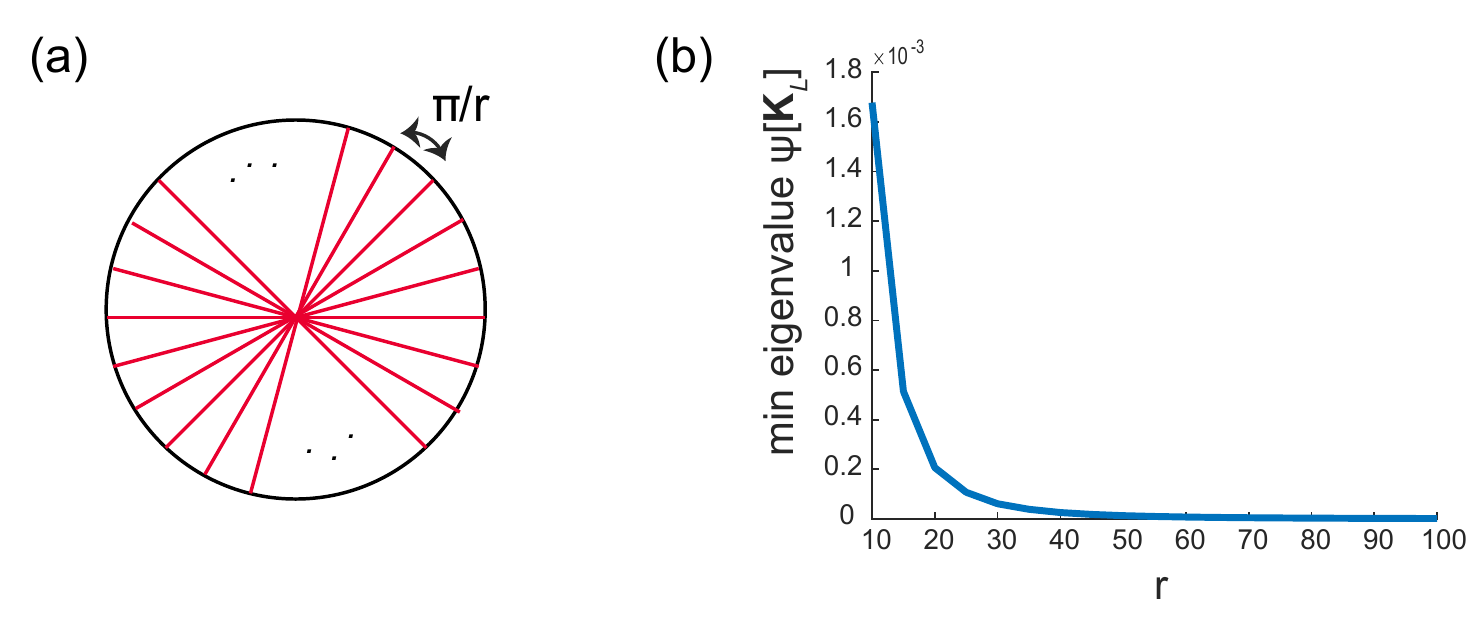}
\caption{(a) An example of $\cL$ in a two-dimensional space such that angles between adjacent lines are equal to one another. (b) The minimum eigenvalue of the matrix $\psi[\bK_{\cL}]$ for different values of $r$.}
\label{fig:min-eig}
\end{figure}

\begin{example}
\textup{Let $\cL=\{L_1,L_2,...,L_r\}$ contain lines in $\mathbb{R}^2$ such that angles between adjacent lines are equal to $\pi/r$ (Figure \ref{fig:min-eig}-a). Thus, we have $\bA_{\cL}(i,j)=\pi|i-j|/r$ for $1\leq i,j\leq r$. Figure \ref{fig:min-eig}-b shows the minimum eigenvalue of the matrix $\psi[\bK_{\cL}]$ for different values of $r$. As the number of lines increases, the minimum eigenvalue of $\psi[\bK_{\cL}]$ decreases. However, for a finite value of $r$, the minimum eigenvalue of $\psi[\bK_{\cL}]$ is positive. Thus, in this case, the condition of corollary \ref{cor:general-global} holds. This highlights why considering a discretized neural network function (i.e., finite $r$) facilities characterizing the landscape of the loss function.}
\end{example}

Next, we characterize local optimizers of optimization \eqref{opt:neural-network-PNN} for a general PNN. Define $R(\bs_1,...,\bs_{r})$ as the space of $\bW$ where $\bs_i$ is the sign vector of weights $\bw_j$ over the line $L_i$ (i.e., $j\in \cG_i$).

\begin{theorem}\label{thm:quan-local}
For a general PNN, in regions $R(\bs_1,...,\bs_r)$ where at least $d$ of $\bs_i$'s are not equal to $\pm \mathbf{1}$, every local optimizer of optimization \eqref{opt:neural-network-PNN} is a global optimizers.
\end{theorem}

\begin{proof}
See Section \ref{subsec:proof:quan-local}.
\end{proof}

\begin{example}
\textup{Consider a two-layer PNN with $d$ inputs, $r$ lines and $k$ hidden neurons. Suppose every line corresponds to $t=k/r$ input weight vectors. If we generate weight vectors uniformly at random over their corresponding lines, for every $1\leq i\leq r$, we have
\begin{align}
\PP[\bs_i=\pm \mathbf{1}]=2^{1-t}.
\end{align}
As $t$ increases, this probability decreases exponentially. According to Theorem \ref{thm:quan-local}, to be in the parameter region without bad locals, the event $\bs_i=\pm \mathbf{1}$ should occur for at most $r-d$ of the lines. Thus, if we uniformly pick a parameter region, the probability of selecting a region without bad locals is
\begin{align}
1-\sum_{i=1}^{d-1} {r \choose i} (1-2^{1-t})^i 2^{(1-t)(r-i)}
\end{align}
which goes to one exponentially as $r\to \infty$.}
\end{example}

In practice the number of lines $r$ is much larger than the number of inputs $d$ (i.e., $r\gg d$). Thus, the condition of Theorem \ref{thm:quan-local} which requires $d$ out of $r$ variables $\bs_i$ not to be equal to $\pm \mathbf{1}$ is likely to be satisfied if we initialize the local search algorithm randomly.

\section{Population Risk Landscapes of Mismatched PNNs}\label{sec:mismatched}
In this section, we characterize the population risk landscape of a mismatched PNN optimization where the model that generates the data and the model used in the PNN optimization are different. We assume that the output variable $y$ is generated using a two-layer PNN with $k^*$ neurons whose weights lie on the set of lines $\cL^*$ with neuron-to-line mapping $\cG^*$. That is
\begin{align}\label{eq:output-mismatch}
y=\sum_{i=1}^{k^*} \text{relu}\left(\left(\bw_i^*\right)^t \bx\right),
\end{align}
where $\bw_i^*$ lie on a line in the set $\cL^*$ for $1\leq i\leq k^*$. The neural network optimization \eqref{opt:neural-network-PNN} is over PNNs  with $k$ neurons over the set of lines $\cL$ with the neuron-to-line mapping $\cG$. Note that $\cL$ and $\cG$ can be different than $\cL^*$ and $\cG^*$, respectively.

Let $r=|\cL|$ and $r^*=|\cL^*|$ be the number of lines in $\cL$ and $\cL^*$, respectively. Let $\bu_i^*$ be the unit norm vector on the line $L_i^*\in \cL^*$ such that $s(\bu_i^*)=1$. Similarly, we define $\bu_i$ as the unit norm vector on the line $L_i\in \cL$ such that $s(\bu_i)=1$. Let $\bU_{\cL}=(\bu_1,...,\bu_r)$ and $\bU_{\cL^*}=(\bu_1^*,...,\bu_r^*)$. Suppose the rank of $\bU_{\cL}$ is at least $d$. Define
\begin{align}\label{eq:corr-mats-mismatch}
\bK_{\cL}&=\bU_{\cL}^t \bU_{\cL}\in\mathbb{R}^{r\times r}\\
\bK_{\cL^*}&=\bU_{\cL^*}^t \bU_{\cL}\in\mathbb{R}^{r^*\times r^*}\nonumber\\
\bK_{\cL,\cL^*}&=\bU_{\cL^*}^t \bU_{\cL^*}\in\mathbb{R}^{r\times r^*}.\nonumber
\end{align}

\begin{theorem}\label{thm:global-mismatch}
The loss function \eqref{eq:loss-function} for a mismatched PNN can be written as
\begin{align}\label{eq:loss-mismatch}
L(\bW)=\frac{1}{4}\|\sum_{i=1}^{k} \bw_i-\sum_{i=1}^{k^*}\bw_i^*\|^2+\frac{1}{4}\bq^t\psi[\bK_{\cL}]\bq+\frac{1}{4}(\bq^*)^t\psi[\bK_{\cL^*}]\bq^*-\frac{1}{2}\bq^t\psi[\bK_{\cL,\cL^*}]\bq^*,
\end{align}
where $\psi(.)$ is defined as in \eqref{eq:psi-def} and $\bq$ and $\bq^*$ are defined as in \eqref{eq:q} using $\cG$ and $\cG^*$, respectively.
\end{theorem}
\begin{proof}
See Section \ref{subsec:proof:mismatch-global}.
\end{proof}
If $\cL=\cL^*$ and $\cG=\cG^*$, the mismatched PNN loss \eqref{eq:loss-mismatch} simplifies to the matched PNN loss \eqref{eq:loss-quan}.
\begin{corollary}\label{cor:loss-mismatch-min}
Let
\begin{align}\label{eq:A}
\bK=\left( {\begin{array}{cc}
   \bK_{\cL} & \bK_{\cL,\cL^*} \\
   \bK_{\cL,\cL^*}^t & \bK_{\cL^*}
  \end{array} } \right)\in\mathbb{R}^{(r+r^*)\times (r+r^*)}.
\end{align}
Then, the loss function of a mismatched PNN can be lower bounded as
\begin{align}\label{eq:schur}
L(\bW)\geq \frac{1}{4} \|\bq^*\|^2 \lambda_{\text{min}}\left(\psi[\bK]/\psi[\bK_{\cL}]\right)
\end{align}
where $\psi[\bK]/\psi[\bK_{\cL}]:=\psi[\bK_{\cL^*}]-\psi[\bK_{\cL^*}]^t \psi[\bK_{\cL}]^{\dagger} \psi[\bK_{\cL^*}]$ is the generalized Schur complement of the block $\psi[\bK_{\cL}]$ in the matrix $\psi[\bK]$.
\end{corollary}
In the mismatched case, the loss at global optima can be non-zero since the model used to generate the data does not belong to the set of training models.

Next, we characterize local optimizers of optimization \eqref{opt:neural-network-PNN} for a mismatched PNN. Similar to the matched PNN case, we define $R(\bs_1,...,\bs_{r})$ as the space of $\bW$ where $\bs_i$ is the vector of sign variables of weight vectors over the line $L_i$.

\begin{theorem}\label{thm:mismatched-local}
For a mismatched PNN, in regions $R(\bs_1,...,\bs_r)$ where at least $d$ of $\bs_i$'s are not equal to $\pm \mathbf{1}$, every local optimizer of optimization \eqref{opt:neural-network-PNN} is a global optimizer. Moreover, in those points we have
\begin{align}\label{eq:schur}
L(\bW^*)&=\frac{1}{4} (\bq^*)^t \left(\psi[\bK]/\psi[\bK_{\cL}]\right)\bq^* \\
&\leq \frac{1}{4} \|\bq^*\|^2 \|\psi[\bK]/\psi[\bK_{\cL}]\|.\nonumber
\end{align}
\end{theorem}
\begin{proof}
See Section \ref{subsec:proof:mismatch-local}.
\end{proof}

When the condition of Theorem \ref{thm:mismatched-local} holds, the spectral norm of the matrix $\|\psi[\bK]/\psi[\bK_{\cL}]\|$ provides an upper-bound on the loss value at global optimizers of the mismatched PNN. In Section \ref{sec:apx}, we study this bound in more detail. Moreover, in Section \ref{sec:badlocals}, we study the case where the condition of Theorem \ref{thm:mismatched-local} does not hold (i.e., the local search method converges to a point in parameter regions where more than $r-d$ of variables $\bs_i$ are equal to $\pm \mathbf{1}$).

To conclude this section, we show that if $\bU_{\cL}$ is a perturbed version of $\bU_{\cL^*}$, the loss in global optima of the mismatched PNN optimization \eqref{opt:neural-network-PNN} is small. This shows a continuity property of the PNN optimization with respect to line perturbations.

\begin{lemma}\label{lemma:kernelschurclosest}
Let $\bK$ is defined as in \eqref{eq:A} where $r=r^*$. Let $\bZ:=\bU- \bU^*$ be the perturbation matrix.
Assume that $\lambda_{\min}\left(\psi\left[\bK_{\cL^*}\right]\right) \geq \delta$. If
\begin{align*}
2\sqrt{r}\|\bZ\|_F + \|\bZ\|_F^2 \leq \frac{\delta}{2},
\end{align*}
then
\begin{align*}
\left\|\psi[\bK]/\psi[\bK_{\cL}]\right\|_2 \leq \left(1+\frac{2r}{\delta}\right)\|\bZ\|_F^2 + 4\sqrt{r}\|\bZ\|_F.
\end{align*}
\end{lemma}
\begin{proof}
See Section \ref{subsec:proof:perturb}.
\end{proof}

\section{PNN Approximations of Unconstrained Neural Networks}\label{sec:apx}
In this section, we study whether an unconstrained two-layer neural network function can be approximated by a PNN. We assume that the unconstrained neural network has $d$ inputs and $k^*$ hidden neurons. This neural network function can also be viewed as a PNN whose lines are determined by input weight vectors to neurons. Thus, in this case $r^*\leq k^*$ where $r^*$ is the number of lines of the original network. If weights are generated randomly, with probability one, $r^*=k^*$ since the probability that two random vectors lie on the same line is zero. Note that lines of the ground-truth PNN (i.e., the unconstrained neural network) are unknowns in the training step. For training, we use a two-layer PNN with $r$ lines, drawn uniformly at random, and $k$ neurons. Since we have relu activation functions at neurons, without loss of generality, we can assume $k=2r$, i.e., for every line we assign two neurons (one for potential weight vectors with positive orientations on that line and the other one for potential weight vectors with negative orientations). Since there is a mismatch between the model generating the data and the model used for training, we will have an approximation error. In this section, we study this approximation error as a function of parameters $d$, $r$ and $r^*$.

\subsection{The PNN Approximation Error Under the Condition of Theorem \ref{thm:mismatched-local}}\label{subsec:apx-random}
Suppose $y$ is generated using an unconstrained two-layer neural network with $k^*$ neurons, i.e., $y=\sum_{i=1}^{k^*}\text{relu}(<\bw_i^*,\bx>)$. In this section, we consider approximating $y$ using a PNN whose lines $\cL$ are drawn uniformly at random. Since these lines will be different than $\cL^*$, the neural network optimization can be formulated as a mismatched PNN optimization, studied in Section \ref{sec:mismatched}. Moreover, in this section, we assume the condition of Theorem \ref{thm:mismatched-local} holds, i.e., the local search algorithm converges to a point in parameter regions where at least $d$ of variables $\bs_i$ are not equal to $\pm \mathbf{1}$. The case that violates this condition is more complicated and is investigated in Section \ref{sec:badlocals}.

Under the condition of Theorem \ref{thm:mismatched-local}, the PNN approximation error depends on both $\|\bq^*\|$ and $\|\psi[\bK]/\psi[\bK_{\cL}]\|$. The former term provides a scaling normalization for the loss function. Thus, we focus on analyzing the later term.

\begin{figure}[t]
\centering
  \includegraphics[width=0.8\linewidth]{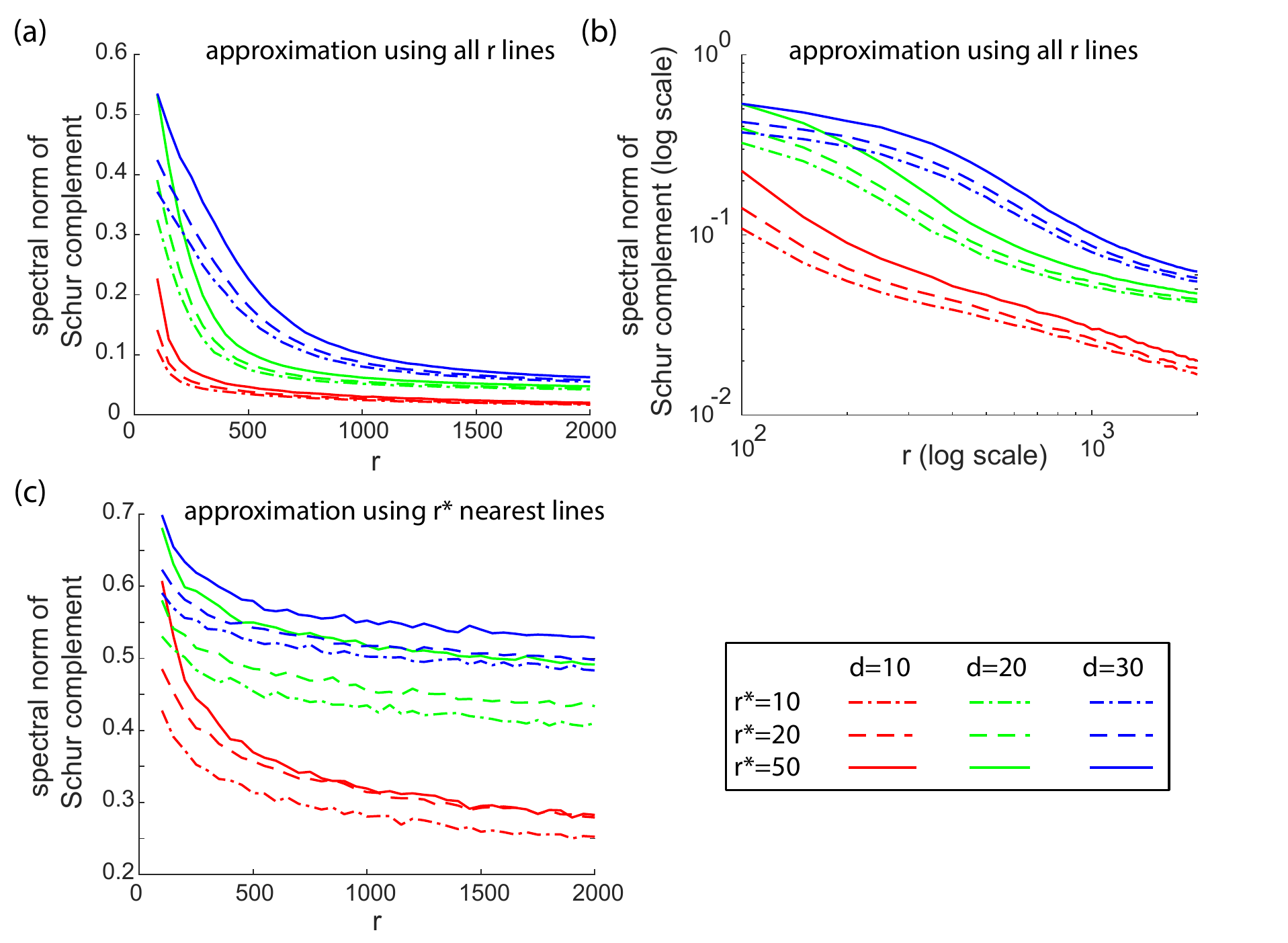}
\caption{(a) The spectral norm of $\psi[\bK]/\psi[\bK_{\cL}]$ for various values of $d$, $r^*$ and $r$. (b) A log-log plot of curves in panel (a). (c) The spectral norm of $\psi[\bK_{nearest}]/\psi[\bK_{\cL_{nearest}}]$ for various values of $d$, $r^*$ and $r$. Experiments have been repeated 100 times. Average results are shown. }
\label{fig:schur-finite-d}
\end{figure}

Since Theorem \ref{thm:mismatched-local} provides an upper-bound for the mismatched PNN optimization loss by $\|\psi[\bK]/\psi[\bK_{\cL}]\|$, intuitively increasing the number of lines in $\cL$ should decrease $\|\psi[\bK]/\psi[\bK_{\cL}]\|$. We prove this in the following theorem.

\begin{theorem}\label{thm:schur-increase-one-line}
Let $\bK$ be defined as in \eqref{eq:A}. We add a distinct line to the set $\cL$, i.e., $\cL_{\text{new}}=\cL\cup L_{r+1}$. Define
\begin{align}\label{eq:A-new}
\bK_{\text{new}}=\left( {\begin{array}{cc}
   \bK_{\cL_{\text{new}}} & \bK_{\cL_{\text{new}},\cL^*} \\
   \bK_{\cL_{\text{new}},\cL^*}^t & \bK_{\cL^*}
  \end{array} } \right)=\left( {\begin{array}{ccc}
   1 & \bz_1^t & \bz_2^t \\
  \bz_1& \bK_{\cL} & \bK_{\cL,\cL^*} \\
  \bz_2& \bK_{\cL,\cL^*}^t & \bK_{\cL^*}
  \end{array} } \right) \in\mathbb{R}^{(r+r^*+1)\times (r+r^*+1)}.
\end{align}
Then, we have
\begin{align}
\|\psi[\bK_{\text{new}}]/\psi[\bK_{\cL_{\text{new}}}]\| \leq \|\psi[\bK]/\psi[\bK_{\cL}]\|.
\end{align}
More specifically,
\begin{align}
\psi[\bK_{\text{new}}]/\psi[\bK_{\cL_{\text{new}}}] = \psi[\bK]/\psi[\bK_{\cL}] - \alpha\bv\bv^t,
\end{align}
where $\alpha = \left(1- \left\langle \psi[\bz_1], \psi[\bK_{\cL}]^{-1}\psi[\bz_1]\right\rangle\right)^{-1} \geq 0$, $\bv = \psi[\bz_2] - \psi[\bK_{\cL,\cL^*}]^t\psi[\bK_{\cL}]^{-1}\psi[\bz_1]$.
\end{theorem}
\begin{proof}
See Section \ref{subsec:proof:add:one:line}.
\end{proof}

Theorem \ref{thm:schur-increase-one-line} indicates that adding lines to $\cL$ decreases $\|\psi[\bK]/\psi[\bK_{\cL}]\|$. However, it does not characterize the rate of this decrease as a function of $r$, $r^*$ and $d$. Next, we evaluate this error decay rate empirically for random PNNs. Figure \ref{fig:schur-finite-d}-a demonstrates the spectral norm of the matrix $\psi[\bK]/\psi[\bK_{\cL}]$ where $\cL$ and $\cL^*$ are both generated uniformly at random. For various values of $r^*$ and $d$, increasing $r$ decreases the PNN approximation error. For moderately small values of $r$, the decay rate of the approximation error is fast. However, for large values of $r$, the decay rate of the approximation error appears to be a polynomial function of $r$ (i.e., the tail is linear in the log-log plot shown in Figure \ref{fig:schur-finite-d}-b). Analyzing $\|\psi[\bK]/\psi[\bK_{\cL}]\|$ as a function of $r$ for fixed values of $d$ and $r$ appears to be challenging. Later in this section, we characterize the asymptotic behaviour of $\|\psi[\bK]/\psi[\bK_{\cL}]\|$ when $d,r\to\infty$.

As explained in Theorem \ref{thm:schur-increase-one-line}, increasing the number of lines in $\cL$ decreases $\|\psi[\bK]/\psi[\bK_{\cL}]\|$. Next, we investigate whether this decrease is due in part to the fact that by increasing $r$, the distance between a subset of $\cL$ with $r^*$ lines and $\cL^*$ decreases. Let $\cL_{nearest}$ be a subset of lines in $\cL$ with $r^*$ lines constructed as follows: for every line $L_i^*$ in $\cL^*$, we select a line $L_j$ in $\cL$ that minimizes $\|\bu_j-\bu_{i}^*\|$ (i.e., $L_j$ has the closest unit vector to $\bu_i$). To simplify notation, we assume that minimizers for different lines in $\cL^*$ are distinct. Using $\cL_{nearest}$ instead of $\cL$, we define $\bK_{nearest}\in \mathbb{R}^{2r^*\times 2r^*}$ as in \eqref{eq:A}.
Figure \ref{fig:schur-finite-d} demonstrates $\|\psi[\bK_{nearest}]/\psi[\bK_{\cL_{nearest}}]\|$ for various values of $r$, $r^*$ and $d$. As it is illustrated in this figure, the PNN approximation error using $r^*$ nearest lines in $\cL$ is significantly larger than the case using all lines.

Next, we analyze the behaviour of $\|\psi[\bK]/\psi[\bK_{\cL}]\|$ when $d,r\to\infty$. There has been some recent interest in characterizing spectrum of inner product kernel random matrices \cite{el2010spectrum,do2012spectrum,cheng2013spectrum,fan2015spectral}. If the kernel is linear, the distribution of eigenvalues of the covariance matrix follows the well-known Marcenko Pastur law. If the kernel is nonlinear, reference \cite{el2010spectrum} shows that in the high dimensional regime where $d,r\to\infty$ and $\gamma=r/d \in(0,\infty)$ is fixed, only the linear part of the kernel function affects the spectrum. Note that the matrix of interest in our problem is the Schur complement matrix $\psi[\bK]/\psi[\bK_{\cL}]$, not $\psi[\bK]$. However, we can use results characterizing spectrum of $\psi[\bK]$ to characterize the spectrum of $\psi[\bK]/\psi[\bK_{\cL}]$.

\begin{figure}[t]
\centering
  \includegraphics[width=0.9\linewidth]{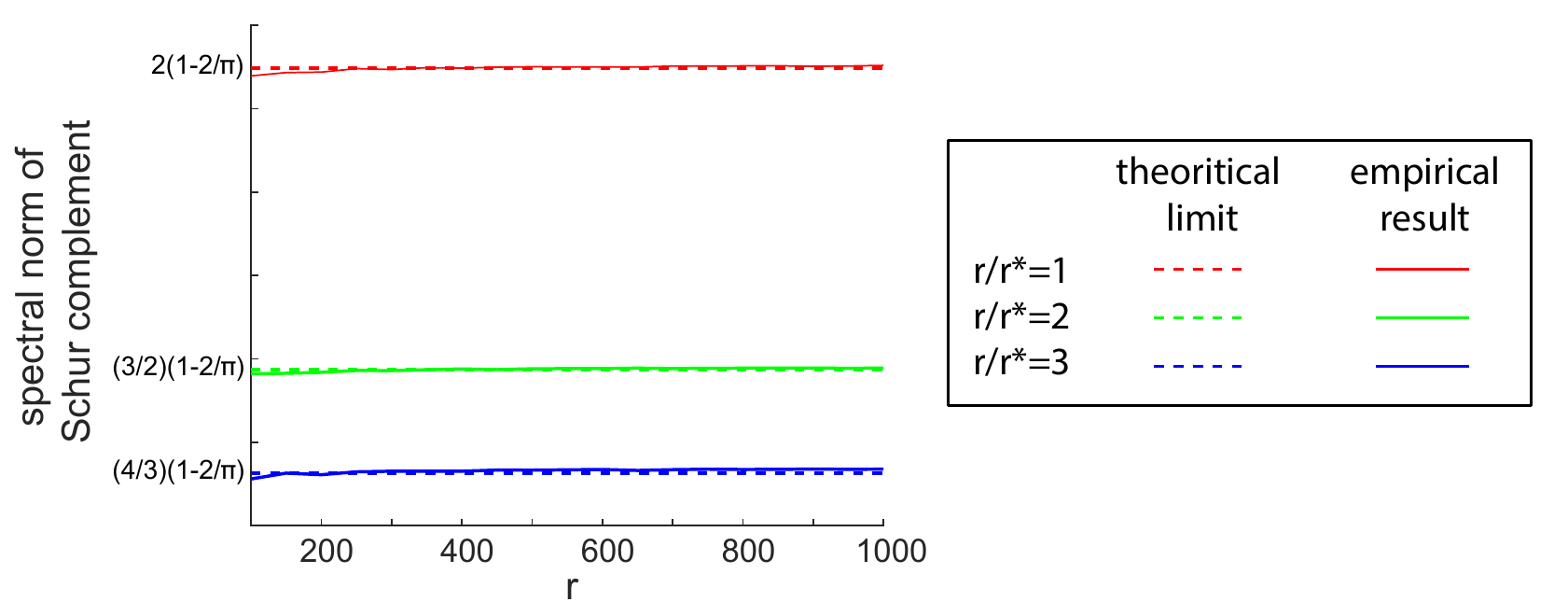}
\caption{The spectral norm of $\psi[\bK]/\psi[\bK_{\cL}]$ when $d=r$. Theoretical limits are described in Theorem \ref{thm:asym}. Experiments have been repeated 100 times. Average results are shown.}
\label{fig:schur-asym-d}
\end{figure}

First, we consider the regime where $r,d\to\infty$ while $\gamma=r/d\in (0,\infty)$ is a fixed number. Theorem 2.1 of reference \cite{el2010spectrum} shows that in this regime and under some mild assumptions on the kernel function (which our kernel function $\psi(.)$ satisfies), $\psi[\bK_{\cL}]$ converges (in probability) to the following matrix:
\begin{align}\label{eq:asym-limit}
\bR_{\cL}=\left(\psi(0)+\frac{\psi''(0)}{2d}\right)\mathbf{1}\mathbf{1}^t+\psi'(0) \bU_{\cL}^t\bU_{\cL}+ \left(\psi(1)-\psi(0)-\psi'(0)\right) \mathbf{I}_{r}.
\end{align}
To obtain this formula, one can write the tailor expansion of the kernel function $\psi(.)$ near 0. It turns out that in the regime where $r,d\to\infty$ while $d/r$ is fixed, it is sufficient for off-diagonal elements of $\psi[\bK_{\cL}]$ to replace $\psi(.)$ with its linear part. However, diagonal elements of $\psi[\bK_{\cL}]$ should be adjusted accordingly (the last term in \eqref{eq:asym-limit}). For the kernel function of our interest, defined as in \eqref{eq:psi-def}, we have $\psi'(0)=0$, $\psi''(0)=2/\pi$, $\psi(0)=2/\pi$ and $\psi(1)=1$. This simplifies \eqref{eq:asym-limit} further to:
\begin{align}\label{eq:asym-limit2}
\bR_{\cL}=\left(\frac{2}{\pi}+\frac{1}{\pi d}\right) \mathbf{1}\mathbf{1}^t+ (1-\frac{2}{\pi}) \mathbf{I}_{r}.
\end{align}
This matrix has $(r-1)$ eigenvalues of $1-2/\pi$ and one eigenvalue of $(2/\pi)r+1-2/\pi+\gamma/\pi$. Using this result, we characterize $\|\psi[\bK]/\psi[\bK_{\cL}]\|$ in the following theorem:
\begin{theorem}\label{thm:asym}
Let $\cL$ and $\cL^*$ have $r$ and $r^*$ lines in $\mathbb{R}^d$ generated uniformly at random, respectively. Let $d,r\to\infty$ while $\gamma=r/d\in (0,\infty)$ is fixed. Moreover, $r^*/r=\cO(1)$. Then,
\begin{align}
 \|\psi[\bK]/\psi[\bK_{\cL}]\|\to \left(1+\frac{r^*}{r}\right)\left(1-\frac{2}{\pi}\right),
\end{align}
where the convergence is in probability.
\end{theorem}
\begin{proof}
See Section \ref{subsec:proof:thm:asym}.
\end{proof}

In the setup considered in Theorem \ref{thm:asym}, the dependency of $\|\psi[\bK]/\psi[\bK_{\cL}]\|$ to $\gamma$ is negligible as it is shown in \eqref{eq:gamma-dep}.

Figure \ref{fig:schur-asym-d} shows the spectral norm of $\psi[\bK]/\psi[\bK_{\cL}]$ when $d=r$. As it is illustrated in this figure, empirical results match closely to analytical limits of Theorem \ref{thm:asym}. Note that by increasing the ratio of $r/r^*$, $\|\psi[\bK]/\psi[\bK_{\cL}]\|$ and therefore the PNN approximation error decreases. If $r^*$ is constant, the limit is $1-2/\pi\approx 0.36$.

Theorem \ref{thm:asym} provides a bound on $\|\psi[\bK]/\psi[\bK_{\cL}]\|$ in the asymptotic regime. In the following corollary, we use this result to bound the PNN approximation error measured as the MSE normalized by the $L_2$ norm of the output variables (i.e., $L(\bW=0)$).

\begin{proposition}\label{prop:scale}
Let $\bW^*$ be the global optimizer of the mismatched PNN optimization under the setup of Theorem \ref{thm:asym}. Then, with high probability, we have
\begin{align}
\frac{L(\bW^*)}{L(\bW=\mathbf{0})}\leq \left(1+\frac{r^*}{r}\right)\left(1-\frac{2}{\pi}\right).
\end{align}
\end{proposition}

\begin{proof}
See Section \ref{proof:prop:scale}.
\end{proof}

This proposition indicates that in the asymptotic regime with $d$ and $r$ grow with the same rate (i.e., $r$ is a constant factor of $d$), the PNN is able to explain a fraction of the variance of the output variable. In practice, however, $r$ should grow faster than $d$ in order to obtain a small PNN approximation error.

\subsection{The General PNN Approximation Error}\label{sec:badlocals}
In this section, we consider the case where the condition of Theorem \ref{thm:mismatched-local} does not hold, i.e., the local search algorithm converges to a point in a {\it bad} parameter region where more than $r-d$ of $\bs_i$ variables are equal to $\pm \mathbf{1}$. To simplify notation, we assume that the local search method has converged to a region where all $\bs_i$ variables are equal to $\pm \mathbf{1}$. The analysis extends naturally to other cases as well.

Let $\bs=(\bs_1,...,\bs_r)$. Let $\bS$ be the diagonal matrix whose diagonal entries are equal to $\bs$, i.e., $\bS=\text{diag}(\bs)$. Similar to the argument of Theorems \ref{thm:quan-local} and \ref{thm:mismatched-local}, a necessary condition for a point $\bW$ to be a local optima of the PNN optimization is:

\begin{align}\label{eq:cond-local-bad}
\bS \bU_{\cL}^t \left(\sum_{i=1}^{k}\bw_i-\sum_{i=1}^{k^*}\bw_i^*\right)+ \psi[\bK_{\cL}]\bq-\psi[\bK_{\cL,\cL^*}]\bq^*=0.
\end{align}

Under the condition of Theorem \ref{thm:mismatched-local}, we have $\sum_{i=1}^{k}\bw_i-\sum_{i=1}^{k^*}\bw_i^*=\mathbf{0}$, which simplifies this condition.

Using \eqref{eq:cond-local-bad} in \eqref{eq:loss-mismatch}, at local optima in bad parameter regions, we have
\begin{align}\label{eq:bad-loss}
4 L(\bW)=\left(\bq^*\right)^t \psi[\bK]/\psi[\bK_{\cL}] \bq^*+ \bz^t \left(\bI+\bU_{\cL} \bS \psi[\bK_{\cL}]^{-1} \bS \bU_{\cL}^t\right) \bz,
\end{align}
where
\begin{align}
\bz:=\sum_{i=1}^{k}\bw_i-\sum_{i=1}^{k^*}\bw_i^*.
\end{align}
The first term of \eqref{eq:bad-loss} is similar to the PNN loss under the condition of Theorem \ref{thm:mismatched-local}. The second term is the price paid for converging to a point in a bad parameter region. In this section, we analyze this term.

The second term of \eqref{eq:bad-loss} depends on the norm of $\bz$. First, in the following lemma, we characterize $\bz$ in local optima.

\begin{lemma}\label{lem:z-bad-local}
In the local optimum of the mismatched PNN optimization, we have
\begin{align}\label{eq:z}
\bz = -\left(\bU_{\cL}\bS\bS^t\bU_{\cL}^t\right)^{-1}\bU_{\cL}\bS\bigg[&\psi[\bK_{\cL}]\left(\bS\bU_{\cL}^t\bU_{\cL}\bS+\psi[\bK_{\cL}]\right)^\dagger\bS\bU_{\cL}^t\bw_0\\
&+ \left(\psi[\bK_{\cL}]\left(\bS\bU_{\cL}^t\bU_{\cL}\bS+\psi[\bK_{\cL}]\right)^\dagger-\bI\right)\psi[\bK_{\cL,\cL^*}]\bq^*\bigg],\nonumber
\end{align}
where
\begin{align*}
\bw_0 \triangleq \sum_{i=1}^{k^*}\bw_i^*.
\end{align*}
\end{lemma}

\begin{proof}
See Section \ref{proof:z-bad-local}.
\end{proof}

Replacing \eqref{eq:z} in \eqref{eq:bad-loss} gives us the loss function achieved at the local optimum.
In order to simplify the loss expression, without loss of generality, from now on we replace $\bU\bS$ with $\bU$ (note that there is essentially no difference between $\bU_{\cL}\bS$ and $\bU_{\cL}$ as the columns of $\bU_{\cL}\bS$ are the columns of $\bU_{\cL}$ with {\it adjusted} orientations.). Moreover, to simplify the analysis of this section, we make the following assumptions.

\begin{assumption}\label{ass:bad-local}
\textup{Recall that we assume that all $\bs_i$ for $1\leq i\leq r$ are equal to $\pm \mathbf{1}$. Our analysis extends naturally to other cases. Moreover, we assume that $\bw_0=0$. This assumption has a negligible effect on our estimate of the value of the loss function achieved in the local minimum in many cases. For example, when $\bw^*_i$ are i.i.d. $\mathcal N(0, (1/d)\bI)$ random vectors, $\bw_0$ is a $\mathcal N(0,(r^*/d)\bI)$ random vector and therefore $\|\bw_0\|_2 = \Theta(\sqrt{r^*})$. On the other hand, $\|\bq^*\|_2 = \Theta(r^*)$. Hence,
in the case where $r^*$ is large, the value of the loss function in the local minimum is controlled by the terms involving $\|\bq^*\|_2^2$ in \eqref{eq:bad-loss}. Thus, we can ignore the terms involving $\bw_0$ in this regime. Finally,
we assume that $\psi[\bK_{\cL}]$ (and consequently $\bU_{\cL}^t\bU_{\cL}+\psi[\bK_{\cL}]$) is invertible.}
\end{assumption}

\begin{theorem}\label{thm:bad-local}
Under assumptions \ref{ass:bad-local}, in a local minimum of the mismatched PNN optimization, we have
\begin{align}\label{eq:bad-loss2}
L(\bW) = \frac{1}{4} (\bq^*)^t \left(\widetilde\psi[\bK]/\psi[\bK_{\cL}]\right)\bq^*,
\end{align}
where
\begin{align*}
\widetilde\psi[\bK] =
\begin{bmatrix}
    \psi[\bK_{\cL}]+\bU_{\cL}^t\bU_{\cL} & \psi[\bK_{\cL,\cL^*}] \\
    \psi[\bK_{\cL,\cL^*}]^t & \psi[\bK_{\cL^*}]\\
\end{bmatrix}.\\
\end{align*}
\end{theorem}

\begin{proof}
See Section \ref{proof:thm:bad:local}.
\end{proof}

The matrix $\widetilde\psi[\bK]$ has an extra term of $\bU_{\cL}^t\bU_{\cL}$ (i.e., the linear kernel) compared to the matrix $\psi[\bK]$. The effect of this term is the price of converging to a local optimum in a bad region. In the following, we analysis this effect in the asymptotic regime where $r,d\to\infty$ while $r/d$ is fixed.

\begin{theorem}\label{thm:asymp2}
Consider the asymptotic case where $r = \gamma d$, $r^* > d+1$, $\gamma > 1$ and $r, r^*, d\to \infty$.
Assume that $k^*=r^*$ underlying weight vectors $\bw^*_{i} \in \reals^d$ are chosen
uniformly at random in $\reals^d$ while the PNN is trained over $r$ lines drawn uniformly at random in $\reals^d$. Under assumption \ref{ass:bad-local}, at local optima, with probability $1-2\exp(-\mu^2 d)$, we have
\begin{align*}
L(\bW) \leq \frac{1}{4}\left( 1-\frac{2}{\pi}+(1+\sqrt{\gamma}+\mu)^2\frac{r^*}{r}\right) \left\|\bq^*\right\|_2^2,
\end{align*}
where $\mu>1$ is a constant.
\end{theorem}

\begin{proof}
See Section \ref{subsec:proof:thm:asym2}.
\end{proof}

Comparing asymptotic error bounds of Theorems \ref{thm:asym} and \ref{thm:asymp2}, we observe that the extra PNN approximation error because of the convergence to a local minimum at a bad parameter region is reflected in the constant parameter $\mu$, which is negligible if $r^*$ is significantly smaller than $r$.

\subsection{A Minimax Analysis of the Naive Nearest Line Approximation Approach}\label{subsec:minimax}
In this section, we show that every realizable function by a two-layer neural network (i.e., every $f\in \cF$) can be approximated arbitrarily closely using a function described by a two-layer PNN (i.e., $\hat{f}\in \cF_{\cL,\cG}$). We start by the following lemma on the continuity of the relu function on the weight parameter:
\begin{lemma}\label{lemma:relucontinuity}
For the relu function $\phi(.)$, we have the following property
\begin{align*}
\left|\phi\left(\left\langle\bw_1, \bx\right\rangle\right)-\phi\left(\left\langle\bw_2, \bx\right\rangle\right)\right| \leq \left\|\bw_1-\bw_2\right\|_2\|\bx\|_2.
\end{align*}
\end{lemma}

\begin{proof}
See Section \ref{subsec:proof-lemma-relu-cont}.
\end{proof}

Recall that $\bu_i$ is the unit norm vector over the line $\L_i$. Let $\mathcal U = \left\{\bu_1, \bu_2, \dots, \bu_r\right\}\subseteq \reals^d$. Denote the set $\mathcal U^- = \left\{-\bu_1, -\bu_2, \dots, -\bu_r\right\}$.

\begin{definition}
For $\delta\in [0,\pi/2]$, we call $\mathcal U$ an angular $\delta$-net of $\mathcal W$ if
for every $\bw \in \mathcal W$, there exists $\bu \in \mathcal U\cup\mathcal U^-$ such that $\theta_{\bu, \bw}\le \delta$.
\end{definition}

The following lemma indicates the size required for $\mathcal U$ to be an angular $\delta$-net
of the unit Euclidean sphere $S^{n-1}$.
\begin{lemma}\label{lemma:spherenet}
Let $\delta\in [0,\pi/2]$. For the unit Euclidean sphere $S^{n-1}$, there exists an angular $\delta$-net $\mathcal U$, with
\begin{align*}
|\mathcal U| \le \frac{1}{2}\left(1 + \frac{\sqrt{2}}{\sqrt{1-\cos \delta}}\right)^n.
\end{align*}
\end{lemma}
\begin{proof}
See Section \ref{subsec:proof:lem:spherenet}.
\end{proof}

The following is a corollary of the previous lemma.
\begin{corollary}\label{corollary:gammanet}
Consider a two-layer neural network with $s$-sparse weights (i.e.,
$\mathcal W$ is the set of $s$-sparse vectors.). In this case, using lemma \ref{lemma:spherenet}, $\mathcal U$ is
an angular $\delta$-net of $\mathcal W$ with
\begin{align*}
|\mathcal U| = \frac{1}{2}{d \choose s}\left(1 + \frac{\sqrt{2}}{\sqrt{1-\cos \delta}}\right)^s.
\end{align*}
Furthermore, if we know the sparsity patterns of $k$ neurons in the network (i.e., if we know the network architecture),
$\tilde{\mathcal U}$ is an angular $\delta$-net of $\mathcal W$ with
\begin{align*}
|\tilde{\mathcal  U}| \le \frac{k}{2}\left(1 + \frac{\sqrt{2}}{\sqrt{1-\cos \delta}}\right)^s.
\end{align*}
\end{corollary}

In order to have a measure of how accurately a function in $\mathcal F$ can be
approximated by a function in $\cF_{\cL}$, we have the following definition:
\begin{definition}
Define $\mathcal R\left(\mathcal F, \cF_{\cL,\cG}\right)$,
the minimax risk of approximating a function in $\mathcal F$ by
a function in $\cF_{\cL,\cG}$, as the following
\begin{align}\label{eq:Riskdef}
\mathcal R\left(\cF_{\cL,\cG}, \mathcal F\right):= \max_{f\in \mathcal F}\min_{\hat f \in \cF_{\cL,\cG}}\quad \EE\left|f(\bx) - \hat f(\bx)\right|,
\end{align}
where the expectation is over $\bx\sim\mathcal N(0,\bI)$.
\end{definition}

The following theorem bounds this
minimax risk where $\mathcal U$ is an angular $\delta$-net of $\mathcal W$.
\begin{theorem}
\label{thm:minimaxrisk}
Assume that for all $\bw \in \mathcal W$, $\|\bw\|_2 \leq M$. Let $\mathcal U$ be an angular $\delta$-net of $\mathcal W$. The minimax risk of approximating a function in $\cF$ with a function in $\cF_{\cL,\cG}$ defined in \eqref{eq:Riskdef} can be written as
\begin{align*}
\mathcal R\left(\cF_{\cL,\cG}, \mathcal F\right) \le kM\sqrt{2d(1-\cos\delta)}.
\end{align*}
\end{theorem}
\begin{proof}
See Section \ref{subsec:proof:thm:minimax}.
\end{proof}
The following is a corollary of Theorem \ref{thm:minimaxrisk} and Corollary \ref{corollary:gammanet}.
\begin{corollary}
Let $\mathcal F$ be the set of realizable functions by a two-layer neural network with $s$-sparse weights. There exists a set $\cL$ and a neuron-to-line mapping $\cG$ such that
\begin{align*}
\mathcal R\left(\cF_{\cL,\cG}, \mathcal F\right) \le \delta,
\end{align*}
and
\begin{align*}
|\mathcal L| \le \frac{1}{2}{d \choose s}\left(1+\frac{2kM\sqrt{d}}{\delta}\right)^s.
\end{align*}
Further, if we know the sparsity patterns of $k$ neurons in the network (i.e., the network architecture), then
\begin{align*}
|\mathcal L| \le \frac{k}{2}\left(1+\frac{2kM\sqrt{d}}{\delta}\right)^s.
\end{align*}
\end{corollary}

\section{Experimental Results}\label{sec:exp}
In our first experiment, we simulate the degree-one PNNs discussed in section 5.2 \footnote{All experiments were implemented in Python 2.7 using the TensorFlow package.}. In the matched case, we are interested in how often we achieve zero loss when we learn using the same network architecture used to generate data (i.e., $\mathcal{L} = \mathcal{L}^*$, $\cG=\cG^*$). We implement networks with $d$ inputs, $k$ hidden neurons, and a single output. Each input is connected to $k/d$ neurons (we assume $k$ is divisible by $d$.). As described previously in Section \ref{sec:degree-one}, relu activation functions are used only at hidden neurons.

\begin{figure}[t]
\centering
  \includegraphics[width=0.8\linewidth]{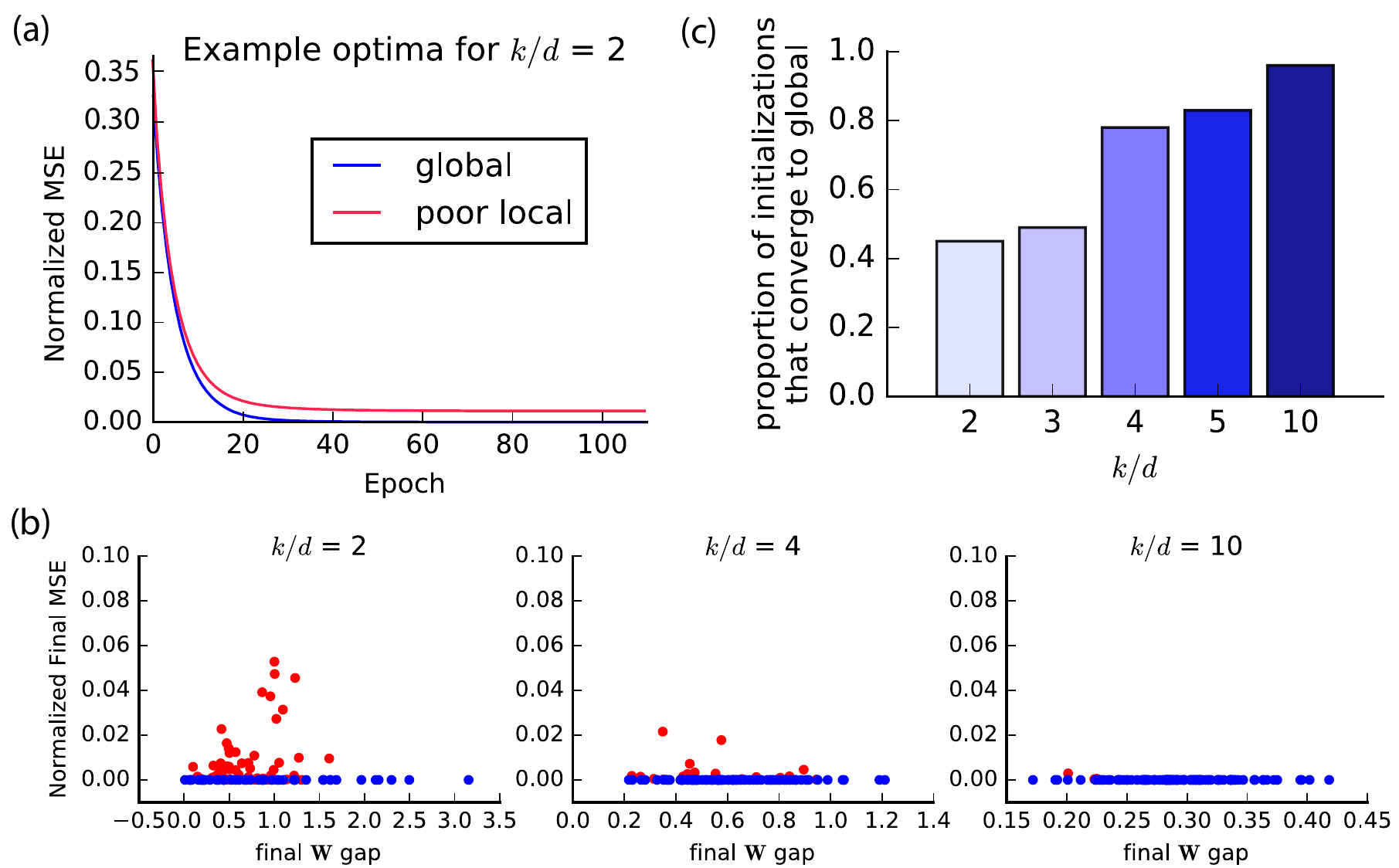}
\caption{(a) The loss during training for two initializations of a degree-one PNN with 5 inputs and 10 hidden neurons.  (b) Plots of the final loss with respect to the gap between the true and estimated weights for different values of $k$. The gap is defined as the Frobenius norm squared of the difference. 100 initializations were used for each value of $k$. (c) A bar plot showing the proportion of global optima found for different values of $k$.}
\label{fig:sim1}
\end{figure}

We use $d = 5$, and $k = 10, 15, 20, 25, 50$. For each value of $k$, we perform 100 trials of the following:
\begin{itemize}
\item [-] Randomly choose a ground truth set of weights.
\item [-] Generate 10000 input-output pairs using the ground truth set of weights.
\item [-] Randomly choose a new set of weights for initialization.
\item [-] Train the network via stochastic gradient descent using batches of size 100, 1000 training epochs, a momentum parameter of 0.9, and a learning rate of 0.01 which decays every epoch at a rate of 0.95 every 390 epochs. Stop training early if the mean loss for the most recent 10 epochs is below $10^{-5}$.
\end{itemize}
As shown in Figure \ref{fig:sim1}-a, we observe that some initializations converge to zero loss while other initializations converge to local optima. Figure \ref{fig:sim1}-b illustrates how frequently an initialized network manages to find the global optimum. We see that as $k$ increases, the probability of finding a global optimum increases. We also observe that for all local optima, $\mathbf{s}_i = \pm \mathbf{1}$ for at least one hidden neuron $i$. In other words, for at least one hidden neuron, all $d$ weights shared the same sign. This is consistent with Theorem \ref{thm:one-degree-local}. Figure \ref{fig:sim1}-c provides a summary statistics of the proportion of global optima found for different values of $k$.

\begin{figure}
\centering
  \includegraphics[height=0.55\linewidth]{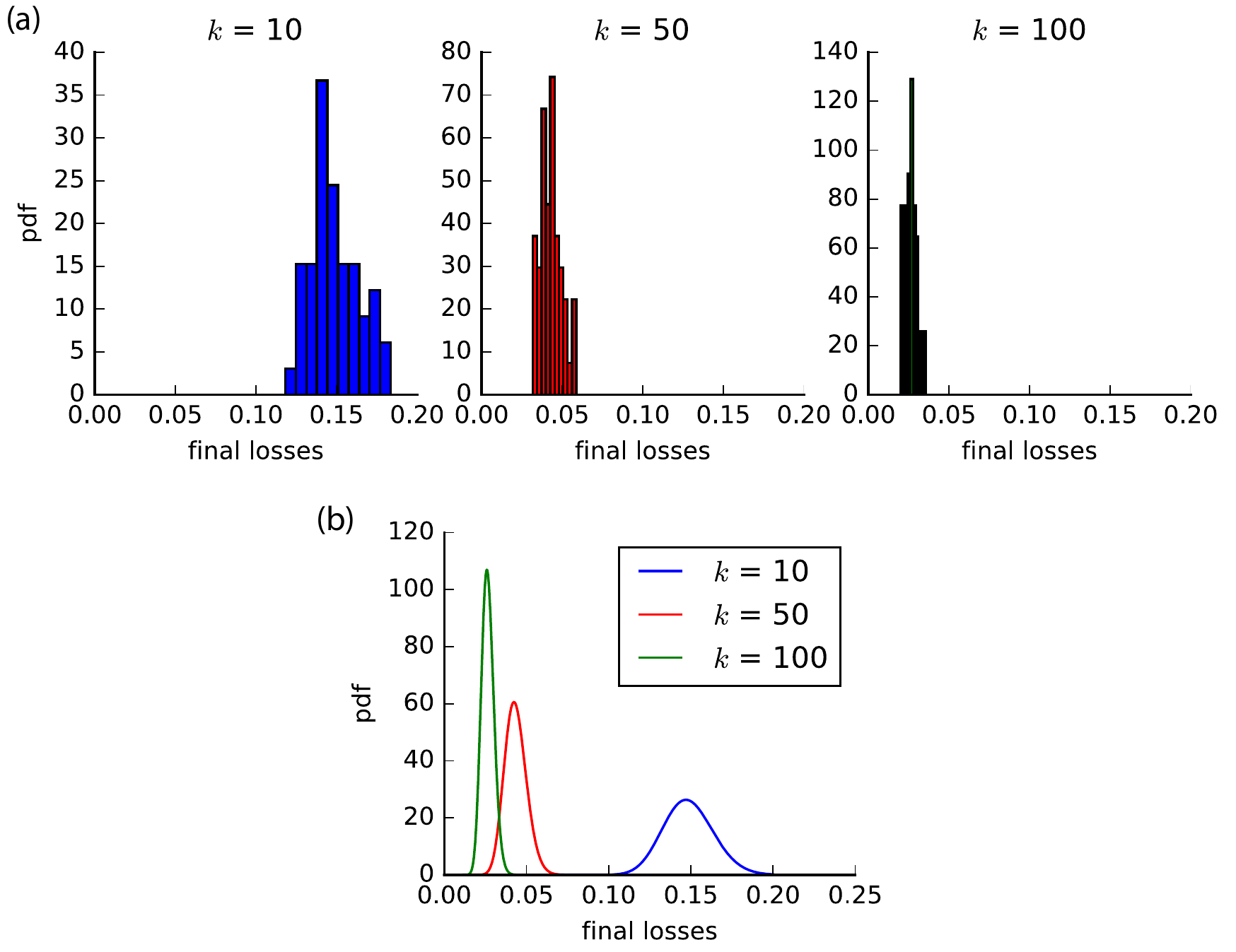}
\caption{(a) Histograms of final losses (i.e., PNN approximation errors) for different values of $k$. (b) Gamma curves fit to the histograms of panel (a).}
\label{fig:sim2}
\end{figure}

Next, we numerically simulate random PNNs in the mismatched case as described in Section \ref{sec:mismatched}. To enforce the PNN architecture, we project gradients along the directions of PNN lines before updating the weights. For example, if we consider $\bw_i^{(0)}$ as the initial set of $d$ weights connecting hidden neuron $i$ to the $d$ inputs, then the final set of weights $\bw_i^{(T)}$ need to lie on the same line as $\bw_i^{(0)}$. To guarantee this, before applying gradient updates to $\bw_i$, we first project them along $\bw_i^{(0)}$.

For PNNs, we use $10 \leq k\leq 100$ hidden neurons. For each value of $k$, we perform 25 trials of the following:
\begin{itemize}
\item [1.] Generate one set of true labels using a fully-connected two-layer network with $d = 15$ inputs and $k^* = 20$ hidden neurons. Generate 10,000 ground training samples and 10,000 test samples using a set of randomly chosen weights.
\item [2.] Initialize $k/2$ random $d$-dimensional unit-norm weight vectors.
\item [3.] Assign each weight vector to two hidden neurons. For the first neuron, scale the vector by a random number sampled uniformly between 0 and 1. For the second neuron, scale the vector by a random number sampled uniformly between -1 and 0.
\item [4.] Train the network via stochastic gradient descent using batches of size 100, 100 training epochs, no momentum, and a learning rate of $10^{-3}$ which decays every epoch at a rate of 0.95 every 390 epochs.
\item [5.] Check to make sure that final weights lie along the same lines as initial weights. Ignore results if this is not the case due to numerical errors.
\item [6.] Repeat steps 2-5 10 times. Return the normalized MSE (i.e., MSE normalized by the $L_2$ norm of $y$) in the test set over different initializations.
\end{itemize}

The results are shown in Figures \ref{fig:PNN-intro} and \ref{fig:sim2}. Figure \ref{fig:sim2}-a shows that as $k$ increases, the PNN approximation gets better in consistent to our theoretical results in Section \ref{sec:apx}. Figure \ref{fig:sim2}-b shows the result of fitting gamma curves to the histograms. We can observe that the curve being compressed towards smaller loss values as $k$ increases.

\section{Conclusion and Future Work}\label{sec:conc}
In this paper, we introduced a family of constrained neural networks, called Porcupine Neural Networks (PNNs), whose population risk landscapes have good theoretical properties, i.e., most local optima of PNN optimizations are global while we have a characterization of parameter regions where bad local optimizers may exist. We also showed that an unconstrained (fully-connected) neural network function can be approximated by a polynomially-large PNN. In particular, we provided approximation bounds at global optima and also bad local optima (under some conditions) of the PNN optimization. These results may provide a tool to explain the success of local search methods in solving the unconstrained neural network optimization because every bad local optimum of the unconstrained problem can be viewed as a local optimum (either good or bad) for a PNN constrained problem where our results provide a bound for the loss value. We leave further explorations of this idea for future work. Moreover, extensions of PNNs to network architectures with more than two layers, to networks with different activation functions, and to other neural network families such as convolutional neural networks are also among interesting directions for future work.

\section{Acknowledgments}
We would like to thank Prof. Andrea Montanari for helpful discussions regarding the PNN approximation bound.

\section{Proofs}
\subsection{Preliminary Lemmas}
\begin{lemma}\label{lem:cov-truncated}
Let $\bx\sim \cN(0,\bI)$. We have
\begin{align}
 \EE\left[\mathbf{1}\{\bw_1^t \bx>0,\bw_2^t \bx>0\}\bx \bx^t\right]=\frac{\pi-\theta_{\bw_1,\bw_2}}{2\pi}\bI+\frac{\sin\left(\theta_{\bw_1,\bw_1}\right)}{2\pi} \bM(\bw_1,\bw_2),
\end{align}
where
\begin{align}
\bM(\bw_1,\bw_2)\triangleq \frac{1}{\sin\left(\theta_{\bw_1,\bw_2}\right)^2} (\bw_1, \bw_2)  \left( {\begin{array}{cc}
   -\cos\left(\theta_{\bw_1,\bw_2}\right) & 1 \\
   1 & -\cos\left(\theta_{\bw_1,\bw_2}\right) \\
  \end{array} } \right)  (\bw_1, \bw_2)^t.
\end{align}
\end{lemma}
\begin{proof}
See e.g., equation (21) in reference \cite{giryes2016deep}.
\end{proof}

Note that $\bM(\bw_1,\bw_2)\bw_1=\frac{\|\bw_1\|}{\|\bw_2\|}\bw_2$, $\bM(\bw_1,\bw_2)\bw_2=\frac{\|\bw_2\|}{\|\bw_1\|}\bw_1$, and $\bM(\bw_1,\bw_2)\bv=0$ for every vector $\bv \perp \text{span}(\bw_1,\bw_2)$.

\begin{lemma}\label{lem:scalar}
Let $x\sim \cN(0,1)$. We have
\begin{align}
 \EE\left[\mathbf{1}\{w_1 x>0,w_2 x>0\}x^2  \right]=\frac{1+s(w_i)s(w_j)}{4}.
\end{align}
\end{lemma}

\begin{lemma}\label{lemma:schurclosest}
Consider
\begin{align*}
\bM = \begin{bmatrix}
        \bA          & \bA + \bDelta_1 \\[0.3em]
       \bA^t + \bDelta_1^t & \bA + \bDelta_2
     \end{bmatrix}\succeq 0
\end{align*}
where $\|\bDelta_1\|_2 \leq \sigma_1$, $\|\bDelta_2\|_2 \leq \sigma_2$ and $\lambda_{\min}\left(\bA\right) \geq \delta$.
Then
\begin{align*}
\|\bM/\bA\|_2 \leq \frac{\sigma_1^2}{\delta} + 2\sigma_1 + \sigma_2.
\end{align*}
\end{lemma}
\begin{proof}
Note that
\begin{align*}
\bM/\bA &= \bA+\bDelta_2 - \left(\bA+\bDelta_1^t\right)\bA^{-1}\left(\bA+\bDelta_1\right)\\
&= \bA + \bDelta_2 - \bA -\bDelta_1^t-\bDelta_1-\bDelta_1^t\bA^{-1}\bDelta_1\\
&= \bDelta_2 -\bDelta_1^t-\bDelta_1-\bDelta_1^t\bA^{-1}\bDelta_1.
\end{align*}
Hence,
\begin{align*}
\|\bM/\bA\|_2 &= \|\bDelta_2 -\bDelta_1^t-\bDelta_1-\bDelta_1^t\bA^{-1}\bDelta_1\|_2\\
&\leq \|\bDelta_1^t\|_2\|\bA^{-1}\|_2\|\bDelta_1\|_2 + 2\|\bDelta_1\|_2 + \|\bDelta_2\|_2\\
&\leq \frac{\sigma_1^2}{\delta} + 2\sigma_1 + \sigma_2.
\end{align*}
\end{proof}

\begin{lemma}\label{lem:inversion-perturbed}
Suppose $\lambda_{\text{min}}(\bA)\geq c>0$ for some $c$. Then, for sufficiently small $\|\bDelta\|$, we have
\begin{align}
 (\bA+\bDelta)^{-1}-\bA^{-1}=\bA^{-1}\tilde{\bDelta}\bA^{-1}
\end{align}
where $\|\tilde{\bDelta}\|\leq 2 \|\bDelta\|$.
\end{lemma}
\begin{proof}
From \cite{chang2006inversion}, we have
\begin{align}
(\bA+\bDelta)^{-1}=\bA^{-1}-\bA^{-1}\bDelta(\bI+\bA^{-1}\bDelta)^{-1} \bA^{-1}.
\end{align}
Let
\begin{align}
\tilde{\bDelta}:=-\bDelta (\bI+\bA^{-1}\bDelta)^{-1}.
\end{align}
Thus, we have
\begin{align}\label{eq:perurb1}
\|\tilde{\bDelta}\|&\leq \|\bDelta\| \|(\bI+\bA^{-1}\bDelta)^{-1}\|\\
&=\|\bDelta\|\frac{1}{\lambda_{\text{min}}(\bI+\bA^{-1}\bDelta)}.\nonumber
\end{align}
Moreover, if $\|\bDelta\|\leq c/2$, we have
\begin{align}\label{eq:perurb2}
\lambda_{\text{min}}(\bI+\bA^{-1}\bDelta)&\geq 1-\|\bA^{-1}\bDelta\|\\
&\geq 1-\frac{\|\bDelta\|}{\lambda_{\text{min}}(\bA)}\geq \frac{1}{2}.
\end{align}
Using \eqref{eq:perurb1} and \eqref{eq:perurb2}, for $\|\bDelta\|\leq c/2$, we have $\|\tilde{\bDelta}\|\leq 2 \|\bDelta\|$. This completes the proof.
\end{proof}

\begin{lemma}\label{lem:inverse-kernel}
Let $\bA=\alpha_1 \bI_n+\beta_1 \mathbf{1}_{n}$. Then
\begin{align}
\bA^{-1}=\alpha_2 \bI_n+ \beta_2 \mathbf{1}_{n},
\end{align}
where
\begin{align}
  \alpha_2&=\frac{1}{\alpha_1} \\
  \beta_2&=-\frac{-\beta_1}{\alpha_1^2+\alpha_1\beta_1 n}.\nonumber
\end{align}
\end{lemma}

\subsection{Proof of Theorem \ref{thm:scalar-global}}\label{subsec:proof-thm-global}
In this case, we can re-write $L(\bW)$ as follows:
\begin{align}\label{eq:all-terms}
L(\bW)&=\EE\left[\left(\sum_{i=1}^{k} \mathbf{1}\{w_i x>0\} w_i x-\sum_{i=1}^{k} \mathbf{1}\{w_i^* x>0\} w_i^* x\right)^2\right]\\
&=\EE\left[\left(\sum_{i=1}^{k} \mathbf{1}\{w_i x>0\} w_i x\right)^2\right]+\EE\left[\left(\sum_{i=1}^{k} \mathbf{1}\{w_i^* x>0\} w_i^* x\right)^2\right]\nonumber\\
&-\EE\left[\sum_{i,j}\mathbf{1}\{w_i x>0, w_j^* x>0\} w_i w_j^* x^2\right].\nonumber
\end{align}
The first term of \eqref{eq:all-terms} can be simplified as follows:
\begin{align}\label{eq:term1}
\EE\left[\left(\sum_{i=1}^{k} \mathbf{1}\{w_i x>0\} w_i x\right)^2\right]&=\frac{1}{2}\sum_{i=1}^{k} w_i^2+\frac{1}{4}\sum_{i\neq j} w_i w_j \left(s(w_i)s(w_j)+1\right)\\
&=\frac{1}{4}\left(\sum_{i=1}^{k} w_i^2+\sum_{i\neq j} w_i w_j\right)+ \frac{1}{4}\left(\sum_{i=1}^{k} s(w_i)w_i^2+\sum_{i\neq j} s(w_i)s(w_j)w_iw_j\right)\nonumber\\
&=\frac{1}{4}\left(\sum_{i=1}^{k}w_i\right)^2+\frac{1}{4}\left(\sum_{i=1}^{k}s(w_i)w_i\right)^2,\nonumber
\end{align}
where the first step follows from Lemma \ref{lem:scalar}. The second term of \eqref{eq:all-terms} can be simplified similarly. The third term of \eqref{eq:all-terms} can be re-written as
\begin{align}\label{eq:third-term}
\EE\left[\sum_{i,j}\mathbf{1}\{w_i x>0, w_j^* x>0\} w_i w_j^* x^2\right]&=\frac{1}{4}\sum_{i,j} w_i w_j^* (s(w_i)s(w_j^*)+1)\\
&=\frac{1}{4} \left(\sum_{i,j} w_i w_j^* \right)+ \frac{1}{4} \left(\sum_{i,j} s(w_i) s(w_j^*) w_i w_j^* \right).
\end{align}
Substituting \eqref{eq:term1} and \eqref{eq:third-term} in \eqref{eq:all-terms}, we have
\begin{align}\label{eq:loss-simple}
 L(\bW)=&\frac{1}{4} \left( (\sum_{i=1}^{k} w_i)^2+ (\sum_{i=1}^{k} w_i^*)^2-(\sum_{i,j} w_i w_j^*)\right)^2\\
 +&\frac{1}{4} \left( (\sum_{i=1}^{k} s(w_i) w_i)^2+ (\sum_{i=1}^{k} s(w_i^*)w_i^*)^2-(\sum_{i,j} s(w_i)s(w_j^*)w_i w_j^*) \right)^2\nonumber\\
 =&\frac{1}{4}\left(\sum_{i=1}^{k} w_i-\sum_{i=1}^{k} w_i^*\right)^2+\frac{1}{4}\left(\sum_{i=1}^{k} s(w_i)w_i-\sum_{i=1}^{k} s(w_i^*)w_i^*\right)^2.
\end{align}
Therefore, $L(\bW)=0$ if and only if $\sum_{i=1}^{k} w_i= \sum_{i=1}^{k} w_i^*$ and $\sum_{i=1}^{k} s(w_i) w_i= \sum_{i=1}^{k} s(w_i^*)w_i^*$. This completes the proof.

\subsection{Proof of Theorem \ref{thm:local}}\label{subsec:proof-thm-local}
First, we characterize the gradient of the loss function with respect to $w_j$:
\begin{align}\label{eq:gradient}
\bigtriangledown_{w_j} L(\bW)&=2\EE\left[\left(\sum_{i=1}^{k} \mathbf{1}\{w_i x>0\} w_i x-\sum_{i=1}^{k} \mathbf{1}\{w_i^* x>0\} w_i^* x\right)\left(\mathbf{1}\{w_j x>0\} x\right)\right]\\
&=\frac{1}{2}\sum_{i=1}^{k} w_i w_j \left(1+s(w_i)s(w_j)\right)-\frac{1}{2}\sum_{i=1}^{k} w_i^* w_j \left(1+s(w_i^*)s(w_j)\right)\nonumber\\
&=\frac{1}{2} \left(\sum_{i=1}^{k}w_i-\sum_{i=1}^{k}w_i^*\right)+\frac{s(w_j)}{2} \left(\sum_{i=1}^{k}s(w_i)w_i-\sum_{i=1}^{k}s(w_i^*)w_i^*\right),
\end{align}
where the first step follows from Lemma \ref{lem:scalar}. A necessary condition to have $\bW$ as a local optimizer is $\bigtriangledown_{w_j} L(\bw)=0$ for every $j$.

Consider a region $R(\bs)$ where $\bs\neq \pm \mathbf{1}$. Thus, there are two indices $j_1$ and $j_2$ such that $s(w_{j_1})>0$ and $s(w_{j_2})<0$. To have a local optimizer in this region, we need to have

\begin{align}
&\left(\sum_{i=1}^{k}w_i-\sum_{i=1}^{k}w_i^*\right)+s(w_{j_1}) \left(\sum_{i=1}^{k}s(w_i)w_i-\sum_{i=1}^{k}s(w_i^*)w_i^*\right)=0,\\
&\left(\sum_{i=1}^{k}w_i-\sum_{i=1}^{k}w_i^*\right)+s(w_{j_2}) \left(\sum_{i=1}^{k}s(w_i)w_i-\sum_{i=1}^{k}s(w_i^*)w_i^*\right)=0.\nonumber
\end{align}
Summing these two equations leads to the following conditions:
\begin{align}
&\sum_{i=1}^{k}w_i-\sum_{i=1}^{k}w_i^*=0,\\
& \sum_{i=1}^{k}s(w_i)w_i-\sum_{i=1}^{k}s(w_i^*)w_i^*=0.\nonumber
\end{align}
On the other hand, Theorem \ref{thm:scalar-global} indicates that if $\bW$ satisfies these conditions, its loss value is equal to zero. Thus, such local optimizers are global optimizers. In regions $R(\pm\mathbf{1})$, to have $\bigtriangledown_{w_j} L(\bW)=0$ for every $j$, we only need to have the condition $\sum_{i=1}^{k}w_i-\sum_{i=1}^{k}w_i^*=0$. In this case, if $s(\bW^*)\neq \pm \mathbf{1}$, we will have bad local optimizers. This completes the proof.

\subsection{Proof of Theorem \ref{thm:hessian-scalar}}\label{subsection:proof-hessian-scalar}
For every $1\leq i,j\leq k$, we have
\begin{align}\label{eq:hessian1}
\bigtriangledown_{w_i,w_j}^2 L(\bW)=2\EE[\mathbf{1}\{w_i x>0,w_j x>0\} x^2]=\frac{s(w_i)s(w_j)}{2}
\end{align}
Let $\bH$ be the Hessian matrix where $\bH(i,j)=\bigtriangledown_{w_i,w_j}^2 L(\bW)$. Thus, in the region $R(\bs)$, we have
\begin{align}\label{eq:hessian2}
\bH=\frac{1}{2}\mathbf{1}+\frac{1}{2}\bs \bs^t.
\end{align}
Note that $\bH$ is positive semidefinite and its rank is equal to two except when $\bs=\pm \mathbf{1}$ in which case its rank is equal to one.

\subsection{Proof of Theorem \ref{thm:degree-one-global}}\label{subsec:proof-degree-one-global}
We can re-write $L(\bW)$ as follows:
\begin{align}\label{eq:all-terms2}
L(\bW)&=\EE\left[\left(\sum_{i=1}^{k} \mathbf{1}\{\bw_i^t \bx>0\} \bw_i^t \bx-\sum_{i=1}^{k} \mathbf{1}\{(\bw_i^*)^t \bx>0\} (\bw_i^*)^t \bx\right)^2\right]\\
&=\EE\left[\left(\sum_{i=1}^{k} \mathbf{1}\{\bw_i^t \bx>0\} \bw_i^t \bx\right)^2\right]+\EE\left[\left(\sum_{i=1}^{k} \mathbf{1}\{(\bw_i^*)^t \bx>0\} (\bw_i^*)^t \bx\right)^2\right]\nonumber\\
&-2\EE\left[\sum_{i,j}\mathbf{1}\{\bw_i^t \bx>0, (\bw_j^*)^t \bx>0\} (\bw_i^t \bx)((\bw_j^*)^t\bx)  \right].\nonumber
\end{align}
The first term can be re-written as
\begin{align}\label{eq:one-deg-1}
\EE\left[\left(\sum_{i=1}^{k} \mathbf{1}\{\bw_i^t \bx>0\} \bw_i^t \bx\right)^2\right]=& \frac{1}{2} \sum_{i=1}^{k} w_i^2 +\frac{1}{4}\sum_{\substack{i\neq j\\ g(i)=g(j)}} \left(w_i w_j+|w_i||w_j|\right)\\
&+\frac{1}{2\pi}\sum_{\substack{i, j\\ g(i)\neq g(j)}} |w_i| |w_j|\nonumber
\end{align}
where the first step follows from Lemma \ref{lem:cov-truncated}. A similar equation can be written for the second term of \eqref{eq:all-terms2}. The third term of \eqref{eq:all-terms2} can be re-written as
\begin{align}\label{eq:one-deg2}
-2\EE\left[\sum_{i,j}\mathbf{1}\{\bw_i^t \bx>0, (\bw_j^*)^t \bx>0\} (\bw_i^t \bx)\left((\bw_j^*)^t\bx\right)  \right]=&-\frac{1}{2} \sum_{\substack{i,j\\ g(i)=g(j)}} \left(w_i w_j^*+|w_i| |w_j^*|\right)\\
&-\frac{1}{\pi}\sum_{\substack{i,j\\ g(i)\neq g(j)}} |w_i| |w_j^*|\nonumber
\end{align}
Substituting \eqref{eq:one-deg-1} and \eqref{eq:one-deg2} in \eqref{eq:all-terms2} we have
\begin{align}
4L(\bW)=\sum_{r=1}^{d}\left(\sum_{i\in \cG_r} w_i-w_i^* \right)^2+ \sum_{r=1}^{d}(q_r-q_r^*)^2+\frac{2}{\pi}\sum_{r\neq t} (q_r-q_r^*)(q_t-q_t^*).
\end{align}
This completes the proof.
\subsection{Proof of Theorem \ref{thm:one-degree-local}}\label{subsec:proof-degree-one-local}
First, we characterize the gradient of the loss function with respect to $\bw_j$:
\begin{align}\label{eq:gradient-one-degree}
\bigtriangledown_{\bw_j} L(\bW)&=2\EE\left[\left(\mathbf{1}\{\bw_j^t \bx>0\} \bx\right) \left(\sum_{i=1}^{k} \mathbf{1}\{\bw_i^t \bx>0\} \bw_i^t \bx-\sum_{i=1}^{k} \mathbf{1}\{(\bw_i^*)^t \bx>0\} (\bw_i^*)^t \bx\right)\right]\\
&=\frac{1}{2}\sum_{\substack{i\\ g(i)=g(j)}} \left(1+s(w_i)s(w_j)\right)\bw_i + \frac{1}{2}\sum_{\substack{i\\ g(i)\neq g(j)}} \left(\bw_i+\frac{2\|\bw_i\|}{\pi\|\bw_j\|}\bw_j\right)\nonumber\\
&-\frac{1}{2}\sum_{\substack{i\\ g(i)=g(j)}} \left(1+s(w_i^*)s(w_j)\right)\bw_i^* - \frac{1}{2}\sum_{\substack{i\\ g(i)\neq g(j)}} \left(\bw_i^*+\frac{2\|\bw_i^*\|}{\pi\|\bw_j\|}\bw_j\right)\nonumber\\
&=\frac{1}{2}\left(\sum_{i=1}^{k}\bw_i-\bw_i^*\right)+\frac{s(w_j)}{2}\left((q_{g(j)}-q_{g(j)}^*)+\frac{2}{\pi}\sum_{r\neq g(j)}(q_r-q_r^*) \right)\be_{g(j)}\nonumber
\end{align}
where the first step follows from Lemma \ref{lem:cov-truncated}. A necessary condition to have $\bW$ as a local optimizer of optimization \eqref{opt:neural-network-PNN} is that the projection gradient is zero for every $j$, i.e., $<\bigtriangledown_{\bw_j} L(\bW),\be_{g(j)}>=0$ for every $j$.

Under the condition of Theorem \ref{thm:one-degree-local}, for every $1\leq r\leq d$, there exists $j_1\neq j_2\in \cG_r$ such that $s(\bw_{j_1})s(\bw_{j_2})=-1$. Thus, summing up \eqref{eq:gradient-one-degree} for $j_1$ and $j_2$, we have
\begin{align}\label{eq:proj-first-term-deg-one}
\be_{r}^t \left(\sum_{i=1}^{k}\bw_i-\bw_i^*\right)=0.
\end{align}
Since this is true for every $1\leq r\leq d$, we have $\sum_{i=1}^{k}\bw_i-\bw_i^*=0$. The second term of \eqref{eq:gradient-one-degree} is a vector with a non-zero element at its $g(j)$ component. Having the first term of \eqref{eq:gradient-one-degree} equal to zero, the second term should be zero in local optimizers. This leads to the set of equations
\begin{align}
\bC (\bq-\bq^*)=0
\end{align}
where $\bC$ is defined in \eqref{eq:C1}. On the other hand, using Theorem \ref{thm:degree-one-global}, having these conditions lead to $L(\bW)=0$. In other words, under the conditions of Theorem \ref{thm:one-degree-local}, every local optimizer is a global optimizer for a one-degree PNN. This completes the proof.

\subsection{Proof of Theorem \ref{thm:global-quan}}\label{subsec:proof:qlobal-quan}
First, we decompose $L(\bW)$ to three terms similar to \eqref{eq:all-terms2}. Then the first term can be re-written as follows:

\begin{align}\label{eq:quan-1}
&\EE\left[\left(\sum_{i=1}^{k} \mathbf{1}\{\bw_i^t \bx>0\} \bw_i^t \bx\right)^2\right]\\
&=\frac{1}{2} \sum_{i=1}^{k} \|\bw_i\|^2 +\sum_{l=1}^{r}\sum_{\substack{i\neq j\\ i,j\in \cG_l}} \frac{1+s(\bw_i)s(\bw_j)}{4}\|\bw_i\|\|\bw_j\|\nonumber\\
&+\sum_{l\neq l'}\sum_{\substack{i\in \cG_l\\ j\in \cG_{l'}}} \left(\frac{(\pi-\theta_{\bw_i,\bw_j}\cos(\theta_{\bw_i,\bw_j}))+\sin(\theta_{\bw_i,\bw_j})}{2\pi} \right)\|\bw_i\|\|\bw_j\|\nonumber\\
&=\frac{1}{2} \sum_{i=1}^{k} \|\bw_i\|^2 +\sum_{l=1}^{r}\sum_{\substack{i\neq j\\ i,j\in \cG_l}} \frac{1+s(\bw_i)s(\bw_j)}{4}\|\bw_i\|\|\bw_j\|\nonumber\\
&+\frac{1}{2\pi}\sum_{l\neq l'}\sum_{\substack{i\in \cG_l\\ j\in \cG_{l'}}} \left(s(\bw_i)s(\bw_j)\cos(\bA_{\cL}(l,l'))\left(\frac{\pi}{2}-\left(\bA_{\cL}(l,l')-\frac{\pi}{2}\right)s(\bw_i)s(\bw_j)\right)+\sin(\bA_{\cL}(l,l'))\right)\|\bw_i\|\|\bw_j\|\nonumber\\
&=\frac{1}{4}\left(\sum_{i=1}^{k}\|\bw_i\|^2+\sum_{i\neq j} <\bw_i,\bw_j>  \right)+\frac{1}{4} \left(\sum_{i=1}^{k}\|\bw_i\|^2+\sum_{i\neq j} \bA_{\cL}(g(i),g(j)) \|\bw_i\|\|\bw_j\| \right)\nonumber
\end{align}
where the first step follows from Lemma \ref{lem:cov-truncated}, and in the second step, we use \eqref{eq:quan-angle-wi-wj}. A similar argument can be mentioned for the second term of \eqref{eq:all-terms2}. The third term of \eqref{eq:all-terms2} can be re-written as
\begin{align}\label{eq:quan-3}
-2\EE\left[\sum_{i,j}\mathbf{1}\{\bw_i^t \bx>0, (\bw_j^*)^t \bx>0\} (\bw_i^t \bx)\left((\bw_j^*)^t\bx\right)  \right]=&-\frac{1}{2}\sum_{l=1}^{r}\sum_{\substack{i, j\\ i,j\in \cG_l}} (1+s(\bw_i)s(\bw_j^*))\|\bw_i\|\|\bw_j\|\\
&+\sum_{l\neq l'}\sum_{\substack{i\in \cG_l\\ j\in \cG_{l'}}} <\bw_i,\bw_j^*>+\bA_{\cL}(l,l') \|\bw_i\|\|\bw_j^*\|\nonumber\\
&=-\frac{1}{2}\sum_{i,j} <\bw_i,\bw_j^*>+\bA_{\cL}(g(i),g(j))\|\bw_i\|\|\bw_j^*\|\nonumber
\end{align}
where we use Lemma \ref{lem:cov-truncated} and equation \eqref{eq:all-terms2}. Substituting \eqref{eq:quan-1} and \eqref{eq:quan-3} in \eqref{eq:all-terms2} completes the proof.

\subsection{Proof of Lemma \ref{lem:C-psd}}\label{subsec:proof-lem-C-PSD}
Note that the matrix $\bK=\cos[\bA_{\cL}]$ is a covariance matrix and thus is positive semidefinite. For the function $\psi(.)$ defined as in \eqref{eq:psi-def}, we have
  \begin{align}
    \frac{\partial^j \psi}{\partial x^j}=
    \begin{cases}
      0, & \text{if}\ $j$ \ \text{is odd} \\
      \frac{2/\pi\prod_{i=1}^{j-2}(2i-1)}{2^{j-2}}, & \text{if}\ $j$ \ \text{is even}
    \end{cases}
  \end{align}
Thus, for every $j\geq 1$, we have $\frac{\partial^j \psi}{\partial x^j}\geq 0$. Using Theorem 4.1 (i) of reference \cite{hiai2009monotonicity} completes the proof.

\subsection{Proof of Theorem \ref{thm:quan-local}}\label{subsec:proof:quan-local}
We characterize the gradient of the loss function with respect to $\bw_j$:
\begin{align}\label{eq:gradient-quan}
\bigtriangledown_{\bw_j} L(\bw)&=2\EE\left[\left(\mathbf{1}\{\bw_j^t \bx>0\} \bx\right) \left(\sum_{i=1}^{k} \mathbf{1}\{\bw_i^t \bx>0\} \bw_i^t \bx-\sum_{i=1}^{k} \mathbf{1}\{(\bw_i^*)^t \bx>0\} (\bw_i^*)^t \bx\right)\right]\\
&=\sum_{l=1}^{r}\Bigg(\sum_{i\in \cG_l}\left( \frac{\pi-\theta_{\bw_i,\bw_j}}{2\pi}\bI+\frac{\sin(\theta_{\bw_i,\bw_j})}{2\pi}\bM(\bw_i,\bw_j) \right)\bw_i \nonumber\\
&-\left( \frac{\pi-\theta_{\bw_i^*,\bw_j}}{2\pi}\bI+\frac{\sin(\theta_{\bw_i^*,\bw_j})}{2\pi}\bM(\bw_i^*,\bw_j) \right)\bw_i^*\Bigg)\nonumber\\
&=\frac{1}{4}\sum_{i=1}^{k} (\bw_i-\bw_i^*)+s(\bw_j) \Bigg(\sum_{l=1}^{r}\sum_{i\in \cG_l} \frac{(\pi/2-\bA_{\cL}(l,g(j)))(\|\bw_i\|-\|\bw_i^*\|)}{2\pi}\bu_l\nonumber\\
&+\frac{\sin(\bA_{\cL}(l,g(j)))(\|\bw_i\|-\|\bw_i^*\|)}{2\pi}\bu_{g(i)}\Bigg)\nonumber
\end{align}
where the first step follows from Lemma \ref{lem:cov-truncated}, and in the second step, we use \eqref{eq:quan-angle-wi-wj}.

A necessary condition to have $\bW$ as a local optimizer is that the projected gradient is zero for every $j$, i.e., $\bu_{g(j)}^t \bigtriangledown_{\bw_j} L(\bW)=0$ for every $j$. Under the conditions of Theorem \ref{thm:quan-local}, over $d$ distinct lines, there exists $j_1\neq j_2\in \cG_r$ such that $s(\bw_{j_1})s(\bw_{j_2})=-1$. Thus, summing up \eqref{eq:gradient-quan} for $j_1$ and $j_2$, we have
\begin{align}
\bu_{r}^t \left(\sum_{i=1}^{k}\bw_i-\bw_i^*\right)=0.
\end{align}
Since this is true for $d$ distinct and thus linearly independent lines, we have $\sum_{i}\bw_i-\bw_i^*=0$. Therefore, the inner product of the second term of \eqref{eq:gradient-quan} with $\bu_{g(j)}$ should be zero in local optimizers. This leads to the following equation:
\begin{align}\label{eq:quan-grad-inner-prod}
\sum_{k=1}^{r}\sum_{i\in\cG_l} \psi[\bK_{\cL}]\left(l,g(j)\right) \left(\|\bw_i\|-\|\bw_i^*\|\right)=\sum_{r=1}^{l}\psi[\bK_{\cL}]\left(l,g(j)\right) \left(q_l-q_l^*\right)=0.
\end{align}
Since this should hold for every $j$, a necessary condition for $\bW$ to be a local optimizer is $\psi[\bK_{\cL}](\bq-\bq^*)=0$. On the other hand, using Theorem \ref{thm:global-quan}, such conditions lead to having $L(\bW)=0$. Therefore, such local optimizers are global optimizers. This completes the proof.
\subsection{Proof of Theorem \ref{thm:global-mismatch}}\label{subsec:proof:mismatch-global}
The proof is similar to the one of Theorem \ref{thm:global-quan}.

\subsection{Proof of Theorem \ref{thm:mismatched-local}}\label{subsec:proof:mismatch-local}
A necessary condition for a point to be a local optimizer is that $\bu_{g(j)}^t \bigtriangledown_{\bw_j} L(\bW)=0$ for every $j$. Similarly to the proof of Theorem \ref{thm:quan-local}, under the condition of Theorem \ref{thm:mismatched-local}, we have $\sum_{i=1}^{k}\bw_i-\sum_{i=1}^{k^*}\bw_{i}^*=0$. This leads to the following equation in local optimizers:
\begin{align}
\psi[\bK_{\cL}]\bq=\psi[\bK_{\cL,\cL^*}]\bq^*.
\end{align}
Replacing this equation in the loss function completes the proof.
\subsection{Proof of Lemma \ref{lemma:kernelschurclosest}}\label{subsec:proof:perturb}
To simplify notations, define
\begin{align*}
\bD= \psi[\bK]= \begin{bmatrix}
        \bD_{11}& \bD_{12} \\[0.3em]
       \bD_{12}^t & \bD_{22}
     \end{bmatrix}\succeq 0
\end{align*}
Note that since $\psi(.)$ has Lipschitz constant $L \leq 1$, we have
\begin{align*}
 \left|\left(\bD_{22} - \bD_{11}\right)_{ij}\right| &\leq  \left|\left(\left(\bU^*\right)^t\bU^* - \bU^t\bU\right)_{ij}\right|\\
&=  \left|\left((\bU+\bZ)^t(\bU+\bZ) - \bU^t\bU\right)_{ij}\right| = \left|\left(\bU^t\bZ + \bZ^t\bU +\bZ^t\bZ\right)_{ij}\right|\\
&\leq \left\| \bU_{.,i}\right\|_2\left\| \bZ_{.,j}\right\|_2 + \left\| \bU_{.,j}\right\|_2\left\| \bZ_{.,i}\right\|_2 + \left\| \bZ_{.,i}\right\|_2\left\| \bZ_{.,j}\right\|_2\\
&\leq \left\| \bZ_{.,j}\right\|_2 + \left\| \bZ_{.,i}\right\|_2 + \left\| \bZ_{.,i}\right\|_2\left\| \bZ_{.,j}\right\|_2,
\end{align*}
where the last step follows from the fact that $\|\bU_{.,i}\|=1$. Hence,
\begin{align}
\label{eq:boundDelta2}
\left\|\bD_{22} - \bD_{11}\right\|_2\leq \left\|\bD_{22} - \bD_{11}\right\|_F \leq 2\sqrt{r}\|\bZ\|_F + \|\bZ\|_F^2.
\end{align}
Similarly,
\begin{align*}
 \left|\left(\bD_{12} - \bD_{11}\right)_{ij}\right| &\leq  \left|\left(\bU^t\bU^* - \bU^t\bU\right)_{ij}\right|\\
&=  \left|\left(\bU^t(\bU+\bZ) - \bU^t\bU\right)_{ij}\right| = \left|\left(\bU^t\bZ\right)_{ij}\right|\\
&\leq \left\| \bU_{.,i}\right\|_2\left\| \bZ_{.,j}\right\|_2 \leq \left\| \bZ_{.,j}\right\|_2.
\end{align*}
Thus,
\begin{align}
\label{eq:boundDelta1}
\left\|\bD_{12} - \bD_{11}\right\|_2\leq \left\|\bD_{12} - \bD_{11}\right\|_F \leq \sqrt{r}\|\bZ\|_F.
\end{align}
Further, note that using \eqref{eq:boundDelta2},
\begin{align}
\label{eq:boundlambdamin}
\lambda_{\min}(\bD_{11}) \geq \lambda_{\min}(\bD_{22}) - \left\|\bD_{22} - \bD_{11}\right\|_2 \geq \delta - 2\sqrt{r}\|\bZ\|_F - \|\bZ\|_F^2 \geq \frac{\delta}{2},
\end{align}
under the assumptions of the Lemma. Hence, combining \eqref{eq:boundDelta2}, \eqref{eq:boundDelta1}, \eqref{eq:boundlambdamin}, using Lemma \ref{lemma:schurclosest}, we have
\begin{align*}
\|\bD/\bD_{11}\|_2 &\leq \frac{2\left\|\bD_{12} - \bD_{11}\right\|_2^2}{\delta} + 2\left\|\bD_{12} - \bD_{11}\right\|_2 + \left\|\bD_{22} - \bD_{11}\right\|_2\\
&\leq \left(1+\frac{2r}{\delta}\right)\|\bZ\|_F^2 + 4\sqrt{r}\|\bZ\|_F.
\end{align*}

\subsection{Proof of Theorem \ref{thm:schur-increase-one-line}}\label{subsec:proof:add:one:line}
To simplify notations, we define
\begin{align*}
\bD=\psi[\bK_{\text{new}}]=\psi\left[\left( {\begin{array}{ccc}
   1 & \bz_1 & \bz_2 \\
  \bz_1^t& \bK_{\cL} & \bK_{\cL,\cL^*} \\
  \bz_2^t& \bK_{\cL,\cL^*}^t & \bK_{\cL^*}
  \end{array} } \right)\right]= \left( {\begin{array}{ccc}
   1 & \bzeta_1^t & \bzeta_2^t \\
  \bzeta_1& \bD_{11} & \bD_{12} \\
  \bzeta_2& \bD_{12}^t & \bD_{22}
  \end{array} } \right)
\end{align*}
and
\begin{align*}
&\bR_1 = \bD_{22} - \bD_{12}^t\bD_{11}^{-1}\bD_{12}\,\,,\\
& \bR_2 = \bD \Bigg/ \begin{bmatrix}
        1          & \bzeta_1^t \\[0.3em]
        \bzeta_1 & \bD_{11}
     \end{bmatrix}.
\end{align*}
Note that since $\bD$ is positive semidefinite (Lemma \ref{lem:C-psd}), we have
\begin{align*}
\begin{bmatrix}
1 & \bzeta_1^t \\[0.3em]
\bzeta_1 & \bD_{11}
\end{bmatrix} \succeq 0.
\end{align*}
Hence
\begin{align*}
\begin{bmatrix}
1 & \bzeta_1^t \\[0.3em]
\bzeta_1 & \bD_{11}
\end{bmatrix}\Big / \bD_{11} = 1- \left\langle \bzeta_1, \bD_{11}^{-1}\bzeta_1\right\rangle \geq 0.
\end{align*}
We have
\begin{align*}
\bR_2 &= \bD_{22} - \left[ \bzeta_2  \;\;\;\;\; \bD_{12}^t \right]\,
\begin{bmatrix}
1 & \bzeta_1^t \\[0.3em]
\bzeta_1 & \bD_{11}
\end{bmatrix}^{-1}\,
\begin{bmatrix}
\bzeta_2^t\\[0.3em]
\bD_{12}
\end{bmatrix}\\
&= \bD_{11} - \left[ \bzeta_2  \;\;\;\;\; \bD_{12}^t \right]\,
\begin{bmatrix}
\left(1- \left\langle \bzeta_1, \bD_{11}^{-1}\bzeta_1\right\rangle\right)^{-1} & -\bzeta_1^t\bD_{11}^{-1}\left(1- \left\langle \bzeta_1, \bD_{11}^{-1}\bzeta_1\right\rangle\right)^{-1} \\[0.3em]
-\bD_{11}^{-1}\bzeta_1\left(1- \left\langle \bzeta_1, \bD_{11}^{-1}\bzeta_1\right\rangle\right)^{-1} & \left(\bD_{11}-\bzeta_1\bzeta_1^t\right)^{-1}
\end{bmatrix}\,
\begin{bmatrix}
\bzeta_2^t\\[0.3em]
\bD_{12}
\end{bmatrix}\\
&= \bD_{22} -  \left[ \bzeta_2  \;\;\;\;\; \bD_{12}^t \right]\,
\begin{bmatrix}
\left(1-\left\langle \bzeta_1, \bD_{11}^{-1}\bzeta_1\right\rangle\right)^{-1}\left(\bzeta_1^t - \bzeta_1^t\bD_{11}^{-1}\bD_{12}\right)\\[0.3em]
-\left(1-\left\langle \bzeta_1, \bD_{11}^{-1}\bzeta_1\right\rangle\right)^{-1}\bD_{11}^{-1}\bzeta_1\bzeta_2^t + \left(\bD_{11}-\bzeta_1\bzeta_1^t\right)^{-1}\bD_{12}
\end{bmatrix}\\
&= \bD_{22} + \left(1-\left\langle \bzeta_1, \bD_{11}^{-1}\bzeta_1\right\rangle\right)^{-1}\left[-\bzeta_2\bzeta_2^t + \bzeta_2\bzeta_1^t\bD_{11}^{-1}\bD_{12} + \bD_{12}^t\bD_{11}^{-1}\bzeta_1\bzeta_2^t\right] - \bD_{12}^t\left(\bD_{11}-\bzeta_1\bzeta_1^t\right)^{-1}\bD_{12}.
\end{align*}
Using the Sherman-Morisson formula, we have
\begin{align*}
\left(\bD_{11} - \bzeta_1\bzeta_1^t\right)^{-1} = \bD_{11}^{-1} + \left(1-\left\langle \bzeta_1, \bD_{11}^{-1}\bzeta_1\right\rangle\right)^{-1}\bD_{11}^{-1}\bzeta_1\bzeta_1^t\bD_{11}^{-1}.
\end{align*}
Hence,
\begin{align*}
\bR_2 &= \bD_{22} - \bD_{12}^t\bD_{11}^{-1}\bD_{12} - \left(1-\left\langle \bzeta_1, \bD_{11}^{-1}\bzeta_1\right\rangle\right)^{-1}\left[\bzeta_2\bzeta_2^t - \bzeta_2\bzeta_1^t\bD_{11}^{-1}\bD_{12} - \bD_{12}^t\bD_{11}^{-1}\bzeta_1\bzeta_2^t + \bD_{12}^t\bD_{11}^{-1}\bzeta_1\bzeta_1^t\bD_{11}^{-1}\bD_{12}\right]\\
&= \bR_1 - \left(1-\left\langle \bzeta_1, \bD_{11}^{-1}\bzeta_1\right\rangle\right)^{-1}\left[\left(\bzeta_2 - \bD_{12}^t\bD_{11}^{-1}\bzeta_1\right)\left(\bzeta_2 - \bD_{12}^t\bD_{11}^{-1}\bzeta_1\right)^t\right]\\
&= \bR_1 - \alpha \bv\bv^t
\end{align*}
where $\alpha \geq 0$, $\bv$ are defined in the theorem. Hence, $\bR_1 \succeq \bR_2$ and $\|\bR_1\|_2 \geq \|\bR_2\|_2$. This completes the proof.

\subsection{Proof of Theorem \ref{thm:asym}}\label{subsec:proof:thm:asym}
To simplify notations, define
\begin{align*}
\bD= \psi[\bK]= \begin{bmatrix}
        \bD_{11}& \bD_{12} \\[0.3em]
       \bD_{12}^t & \bD_{22}
     \end{bmatrix}\succeq 0.
\end{align*}
Moreover, let
\begin{align*}
\bR= \begin{bmatrix}
        \bR_{11}& \bR_{12} \\[0.3em]
       \bR_{12}^t & \bR_{22}
     \end{bmatrix},
\end{align*}
where
\begin{align}\label{eq:asym-R}
\bR_{11}&=\alpha\bI_{r_1}+\beta\mathbf{1}_{r_1} \\
\bR_{22}&=\alpha\bI_{r_2}+\beta\mathbf{1}_{r_2}\nonumber \\
\bR_{12}&=\beta\mathbf{1}_{r_1\times r_2},\nonumber
\end{align}
such that
\begin{align*}
\alpha&=1-\frac{2}{\pi}\nonumber\\
\beta&=\frac{2}{\pi}+\frac{1}{\pi d}.\nonumber
\end{align*}
Let
\begin{align*}
\bDelta=\bD-\bR= \begin{bmatrix}
        \bDelta_{11}& \bDelta_{12} \\[0.3em]
       \bDelta_{12}^t & \bDelta_{22}
     \end{bmatrix}.
\end{align*}
Note that to simplify notations, we make the dependency of these matrices to $d$, $r_1$ and $r_2$ implicit. Using Theorem 2.1 of reference \cite{el2010spectrum}, under the assumptions of the theorem, as $d,r_1\to\infty$, we have $\|\bDelta_{11}\|\to 0$, $\|\bDelta_{22}\|\to 0$ and $\|\bDelta_{12}\|\to 0$ in probability. Moreover, we have

\begin{align}\label{eq:asym-pf1}
\bD/\bD_{11}&=\bD_{22}-\bD_{12}^t\bD_{11}^{-1}\bD_{12}\\
&=(\bR_{22}+\bDelta_{22})-(\bR_{12}+\bDelta_{12})^t(\bR_{11}+\bDelta_{11})^{-1}(\bR_{12}+\bDelta_{12}).\nonumber
\end{align}
Since $\lambda_{\text{min}}(\bR_{11})=1-2/\pi$, using Lemma \ref{lem:inversion-perturbed}, we have
\begin{align}
(\bR_{11}+\bDelta_{11})^{-1}=\bR_{11}^{-1}+\bR_{11}^{-1}\tilde{\bDelta_{11}}\bR_{11}^{-1},
\end{align}
where $\|\tilde{\bDelta}\|\to 0$ in probability. Using this equation in \eqref{eq:asym-pf1}, we have
\begin{align}\label{eq:asym-pf2}
\bD/\bD_{11}=\bZ_1+\bZ_2
\end{align}
where
\begin{align}\label{eq:asym-pf3}
\bZ_1=\bR_{22}-\bR_{12}^t\bR_{11}^{-1}\bR_{12}
\end{align}
and
\begin{align}\label{eq:asym-pf4}
\bZ_2&=\bDelta_{22}-\bDelta_{12}^t\bR_{11}^{-1}\bR_{12}-\bDelta_{12}^t\bR_{11}^{-1}\bDelta_{12}\\
&-\bDelta_{12}^t\bR_{11}^{-1}\tbDelta_{11}\bR_{11}^{-1}\bR_{12}-\bDelta_{12}^t\bR_{11}^{-1}\tbDelta_{11}\bR_{11}^{-1}\bDelta_{12}\nonumber\\
&-\bR_{12}^t\bR_{11}^{-1}\bDelta_{12}-\bR_{12}^t\bR_{11}^{-1}\tbDelta_{11}\bR_{11}^{-1}\bR_{12}-\bR_{12}^t\bR_{11}^{-1}\tbDelta_{11}\bR_{11}^{-1}\bDelta_{12}.\nonumber
\end{align}
First, we show that as $d,r_1\to\infty$, $\|\bZ_2\|\to 0$ in probability. Note that using Lemma \ref{lem:inverse-kernel}, we have
\begin{align}\label{eq:exact-inverse-R11}
\bR_{11}^{-1}=\frac{1}{\alpha} \bI_{r_1}-\frac{\beta}{\alpha^2+\alpha\beta r_1} \mathbf{1}_{r_1}.
\end{align}
Therefore, we have
\begin{align}
\mathbf{1}_{r_2\times r_1}\bR_{11}^{-1}=\frac{1}{\alpha+\beta r_1}\mathbf{1}_{r_2\times r_1}.
\end{align}
Thus, we have
\begin{align}\label{eq:asym-pf5}
\| \mathbf{1}_{r_2\times r_1}\bR_{11}^{-1}\|\leq c_1
\end{align}
for sufficiently large $r_1$. Similarly, we have
\begin{align}\label{eq:asym-pf6}
 \|\bR_{11}^{-1}\|\leq c_2,
\end{align}
for sufficiently large $r_1$. Using \eqref{eq:asym-pf5} and \eqref{eq:asym-pf6} in \eqref{eq:asym-pf4}, it is straightforward to show that as $d,r_1\to\infty$, $\|\bZ_2\|\to 0$ in probability.

Next, we characterize $\|\bZ_1\|$. We have
\begin{align}
\bZ_1&=\alpha\bI_{r_2}+\beta \mathbf{1}_{r_2}-\beta^2 \mathbf{1}_{r_2\times r_1} \bR_{11}^{-1}\mathbf{1}_{r_1\times r_2}\\
&=\alpha \bI_{r_2}+\frac{\alpha \beta}{\alpha+\beta r_1}\mathbf{1}_{r_2}.\nonumber
\end{align}
Therefore, we have
\begin{align}\label{eq:gamma-dep}
 \|\bZ_1\|&=\alpha\left(1+\frac{\beta r_2}{\alpha+\beta r_1}\right)\\
 &=\left(1-\frac{2}{\pi}\right)\left(1+\left(1-\frac{\pi-2}{\gamma+\pi-2+2r_1}\right)\frac{r_2}{r_1}\right)\nonumber\\
 &=\left(1-\frac{2}{\pi}\right)\left(1+\frac{r_2}{r_1}\right),\nonumber
 \end{align}
as $r_1\to\infty$. This completes the proof.

\subsection{Proof of Proposition \ref{prop:scale}}\label{proof:prop:scale}
Since $\bq^*$ is a vector in $\mathbb{R}^{r^*}$ whose components are non-negative, we can write
\begin{align}\label{eq:q-proj}
\bq^*=\frac{\|\bq^*\|_{1}}{r^*} \mathbf{1}_{r^*\times 1}+\bq_2^*,
\end{align}
where $\bq_2^*$ is orthogonal to the vector $\mathbf{1}_{r^*\times 1}$. Therefore, we have

\begin{align}\label{eq:w=0-loss}
L(\bW=0)&=\frac{1}{4}\|\sum_{i=1}^{r^*}\bw_i^*\|^2+\frac{1}{4}\left(\bq^*\right)^t \psi[\bK_{\cL^*}] \bq^*\\
&\geq \frac{1}{4} \left(\bq^*\right)^t \left((1-\frac{2}{\pi})\bI_{r^*}+(\frac{2}{\pi}+\frac{1}{\pi d})\mathbf{1}_{r^*\times r^*} \right) \bq^*\nonumber\\
&=\frac{1}{4}(1-\frac{2}{\pi})\|\bq^*\|^2+\frac{1}{2\pi}\|\bq^*\|_{1}^2\nonumber\\
&\geq \frac{1}{4} \|\bq^*\|^2,\nonumber
\end{align}
where the first step follows from Theorem \ref{thm:global-mismatch}, the second step follows from \eqref{eq:asym-R}, the third step follows from \eqref{eq:q-proj} and the fact that $d\to\infty$, and the last step follows from the fact that $\|\bq^*\|_{1}\geq \|\bq^*\|$. Using \eqref{eq:w=0-loss} in Theorem \ref{thm:asym} completes the proof.

\subsection{Proof of Lemma \ref{lem:z-bad-local}}\label{proof:z-bad-local}
To simplify notations, define
\begin{align*}
\bD= \psi[\bK]= \begin{bmatrix}
        \bD_{11}& \bD_{12} \\[0.3em]
       \bD_{12}^t & \bD_{22}
     \end{bmatrix}\succeq 0
\end{align*}
We also use $\bU$ instead of $\bU_{\cL}$.

Let $\bw_j^+ = \sum_{i: \bw_i = \|\bw_i\|\bu_j}\|\bw_i\|$ and
$\bw_j^- = \sum_{i: \bw_i = -\|\bw_i\|\bu_j}\|\bw_i\|$. Thus, we have
\begin{align*}
\bw_j^+ - \bw_j^- = s_jq_j.
\end{align*}
Hence,
\begin{align*}
\sum_{i = 1}^k \bw_i = \sum_{j=1}^{r_1}\left(\bw_j^+-\bw_j^-\right)\bu_j = \sum_{j=1}^{r_1}s_jq_j\bu_j = \bU\bS\bq.
\end{align*}
Therefore, equation \eqref{eq:cond-local-bad} implies that
\begin{align*}
\bS\bU^t\left(\bU\bS\bq - \bw_0\right) + \bD_{11}\bq - \bD_{12}\bq^* = 0.
\end{align*}
Thus,
\begin{align*}
\bq = \left(\bS\bU^t\bU\bS + \bD_{11}\right)^\dagger\left(\bS\bU^t\bw_0+\bD_{12}\bq^*\right)
\end{align*}
and
\begin{align*}
-\bS\bU^t\bz &= \bD_{11}\bq - \bD_{12}\bq^* \\
&= \bD_{11}\left(\bS\bU^t\bU\bS+\bD_{11}\right)^\dagger\bS\bU^t\bw_0 + \left(\bD_{11}\left(\bS\bU^t\bU\bS+\bD_{11}\right)^\dagger-\bI\right)\bD_{12}\bq^*.
\end{align*}
Thus,
\begin{align*}
\bz = -\left(\bU\bS\bS^t\bU^t\right)^{-1}\bU\bS\left[\bD_{11}\left(\bS\bU^t\bU\bS+\bD_{11}\right)^\dagger\bS\bU^t\bw_0 + \left(\bD_{11}\left(\bS\bU^t\bU\bS+\bD_{11}\right)^\dagger-\bI\right)\bD_{12}\bq^*\right].
\end{align*}
\subsection{Proof of Theorem \ref{thm:bad-local}}\label{proof:thm:bad:local}
To simplify notations, define
\begin{align*}
\bD= \psi[\bK]= \begin{bmatrix}
        \bD_{11}& \bD_{12} \\[0.3em]
       \bD_{12}^t & \bD_{22}
     \end{bmatrix}\succeq 0
\end{align*}
We also use $\bU$ instead of $\bU_{\cL}$.

Under assumptions \ref{ass:bad-local}, \eqref{eq:z} simplifies to
\begin{align*}
\bz = -\left(\bU\bU^t\right)^{-1}\bU\bD_{11}\left(\bD_{11}\left(\bU^t\bU+\bD_{11}\right)^{-1}-\bI\right)\bD_{12}\bq^*.
\end{align*}
Using the Woodbury matrix identity,
\begin{align*}
\left(\bD_{11} + \bU^t\bU\right)^{-1} = \bD_{11}^{-1} -\bD_{11}^{-1}\bU^t\left(\bI + \bU\bD_{11}^{-1}\bU^t\right)^{-1}\bU\bD_{11}^{-1}.
\end{align*}
Hence,
\begin{align*}
\bz = \left(\bI + \bU\bD_{11}^{-1}\bU^t\right)^{-1}\bU\bD_{11}^{-1}\bD_{12}\bq^*.
\end{align*}
Therefore,
\begin{align*}
\left\langle \bz, \left(\bI + \bU\bD_{11}^{-1} \bU^t\right)\bz\right\rangle =
\left\langle\bq^*,\bD_{12}^t\bD_{11}^{-1}\bU^t\left(\bI + \bU\bD_{11}^{-1}\bU^t\right)^{-1}\bU\bD_{11}^{-1}\bD_{12}\bq^*\right\rangle.
\end{align*}
Replacing this in \eqref{eq:bad-loss}, we get
\begin{align*}
L(\bW) = \frac{1}{4}\left\langle\bq^*, \left(\bD_{22} - \bD_{12}^t\bD_{11}^{-1/2}\left(\bI - \bD_{11}^{-1/2}\bU^t\left(\bI+\bU\bD_{11}^{-1}\bU^t\right)^{-1}\bU\bD_{11}^{-1}\right)\bD_{11}^{-1/2}\bD_{12}\right)\bq^*\right\rangle.
\end{align*}
Note that we can write
\begin{align*}
\bD_{22} - \bD_{12}^t\bD_{11}^{-1/2}\left(\bI - \bD_{11}^{-1/2}\bU^t\left(\bI+\bU\bD_{11}^{-1}\bU^t\right)^{-1}\bU\bD_{11}^{-1}\right)\bD_{11}^{-1/2}\bD_{12} = \widetilde\bD/\bD_{11},
\end{align*}
where
\begin{align*}
&\widetilde\bD =
\begin{bmatrix}
    \widetilde\bD_{11} & \bD_{12} \\
    \bD_{12}^t & \bD_{22}\\
\end{bmatrix},\\
&\widetilde\bD_{11} = \bD_{11}^{1/2}\left(\bI - \bD_{11}^{-1/2}\bU^t\left(\bI+\bU\bD_{11}^{-1}\bU^t\right)^{-1}\bU\bD_{11}^{-1}\right)^{-1}\bD_{11}^{1/2}.
\end{align*}
Using the Woodbury matrix identity one more time leads to
\begin{align*}
\left(\bI - \bD_{11}^{-1/2}\bU^t\left(\bI+\bU\bD_{11}^{-1}\bU^t\right)^{-1}\bU\bD_{11}^{-1}\right)^{-1} &= \bI - \bD_{11}^{-1/2}\bU^t\left(-\bI-\bU\bD_{11}^{-1}\bU^t+\bU\bD_{11}^t\bU^t\right)\bU\bD_{11}^{-1/2}\\
&= \bI + \bD_{11}^{-1/2}\bU^t\bU\bD_{11}^{-1/2}.
\end{align*}
Thus,
\begin{align*}
\widetilde\bD_{11} = \bD_{11}+ \bU^t\bU, \;\;\;\;
&\widetilde\bD =
\begin{bmatrix}
    \bD_{11}+\bU^t\bU & \bD_{12} \\
    \bD_{12}^t & \bD_{22}\\
\end{bmatrix},\\
\end{align*}
and
\begin{align*}
L(\bW) = \frac{1}{4}\left\langle\bq^*, \left(\widetilde\bD/\bD_{22}\right)\bq^*\right\rangle.
\end{align*}
This completes the proof.

\subsection{Proof of Theorem \ref{thm:asymp2}}\label{subsec:proof:thm:asym2}
To simplify notations, we use $\bU$ instead of the $\bU_{\cL}$. Moreover, we define
\begin{align*}
\psi[\bK]= \begin{bmatrix}
        \bD_{11}& \bD_{12} \\[0.3em]
       \bD_{12}^t & \bD_{22}
     \end{bmatrix},
\end{align*}
and
\begin{align*}
\widetilde\bD_{11} = \bD_{11}+ \bU^t\bU, \;\;\;\;
&\widetilde\bD =
\begin{bmatrix}
    \bD_{11}+\bU^t\bU & \bD_{12} \\
    \bD_{12}^t & \bD_{22}\\
\end{bmatrix}.\\
\end{align*}
Moreover, let
\begin{align*}
\bR= \begin{bmatrix}
        \bR_{11}& \bR_{12} \\[0.3em]
       \bR_{12}^t & \bR_{22}
     \end{bmatrix},
\end{align*}
where
\begin{align}\label{eq:asym-R}
\bR_{11}&=\alpha\bI_{r_1}+\beta\mathbf{1}_{r_1}+\bU^t\bU \\
\bR_{22}&=\alpha\bI_{r_2}+\beta\mathbf{1}_{r_2}\nonumber \\
\bR_{12}&=\beta\mathbf{1}_{r_1\times r_2},\nonumber
\end{align}
such that
\begin{align*}
\alpha&=1-\frac{2}{\pi}\nonumber\\
\beta&=\frac{2}{\pi}+\frac{1}{\pi d}.\nonumber
\end{align*}
Let
\begin{align*}
\bDelta= \bR-\widetilde\bD= \begin{bmatrix}
        \bDelta_{11}& \bDelta_{12} \\[0.3em]
       \bDelta_{12}^t & \bDelta_{22}
     \end{bmatrix}.
\end{align*}
Using the result of Theorem \ref{thm:bad-local}, we have
\begin{align*}
L\left(\bW\right) = \frac{1}{4}\left\langle\bq^*, \left(\widetilde\bD/\bD_{22}\right)\bq^*\right\rangle \leq \frac{1}{4}\left\|\left(\widetilde\bD/\bD_{22}\right)\right\|_2\left\|\bq^*\right\|_2^2.
\end{align*}
Similar to the proof of Theorem \ref{thm:asym}, the $\bDelta$ matrix and the $1/d$ term of $\beta$ have negligible effects in the asymptotic regime. Hence, it is sufficient
to bound $\left\|\bR/\bR_{11}\right\|_2$. We have
\begin{align}\label{eq:setare}
\bR/\bR_{11}=\left(\beta \mathbf{1}_{r_2}+\alpha \bI_{r_2}\right)-\beta^2\left(\beta \mathbf{1}_{r_1}+\alpha \bI_{r_1}+\bU^t\bU\right)^{-1}\mathbf{1}_{r_1\times r_2}.
\end{align}
Note that if $\bu\in\mathbb{R}^{r_2}$ where $\|\bu\|=1$ and $<\bu,\mathbf{1}>=0$, we have
\begin{align}\label{eq:setare0}
\left(\bR/\bR_{11}\right)\bu=\alpha \bu
\end{align}
which leads to $\|(\bR/\bR_{11})\bu\|=\alpha$ and $<\bu,(\bR/\bR_{11}\bu)>=\alpha$. Moreover, we have
\begin{align}\label{eq:setare2}
\lim_{d\to\infty} \frac{1}{r_2} \left<\mathbf{1},(\bR/\bR_{11})\mathbf{1}\right>= \lim_{d\to\infty} \frac{1}{r_2} \left<\mathbf{1}_{r_2},(\bR/\bR_{11})\right>.
\end{align}
Using the Woodbury matrix identity and Lemma \ref{lem:inverse-kernel}, we have
\begin{align}\label{eq:setare3}
&\left(\frac{2}{\pi}\mathbf{1}_{r_1}+\alpha \bI_{r_1}+\bU^t\bU\right)^{-1}=\left(\frac{2}{\pi}+\alpha \bI_{r_1}\right)^{-1}\\
&-\left(\frac{2}{\pi}+\alpha \bI_{r_1}\right)^{-1}\bU^t \left(\bI+\bU \left(\frac{2}{\pi}\mathbf{1}_{r_1}+\alpha \bI_{r_1} \right)^{-1}\bU^t \right)^{-1} \bU \left(\frac{2}{\pi}+\alpha \bI_{r_1}\right)^{-1}\nonumber\\
&=\left(\frac{1}{\alpha}\bI_{r_1}-\frac{2}{\alpha(\pi\alpha+2r_1)}\mathbf{1}_{r_1}\right)\nonumber\\
&-\left(\frac{1}{\alpha}\bI_{r_1}-\frac{2}{\alpha(\pi\alpha+2r_1)}\mathbf{1}_{r_1}\right)\bU^t \left(\bI+\bU \left(\frac{1}{\alpha}\bI_{r_1}-\frac{2}{\alpha(\pi\alpha+2r_1)}\mathbf{1}_{r_1}\right) \bU^t  \right)^{-1}\bU \left(\frac{1}{\alpha}\bI_{r_1}-\frac{2}{\alpha(\pi\alpha+2r_1)}\mathbf{1}_{r_1}\right).\nonumber
\end{align}
Letting
\begin{align}
\bA:= \bU^t \left(\bI+\bU\left(\frac{1}{\alpha}\bI_{r_1}-\frac{2}{\alpha(\pi\alpha+2r_1)}\mathbf{1}_{r_1}\right)\bU^t\right)^{-1}\bU,
\end{align}
we have
\begin{align}
\frac{4}{\pi^2}\mathbf{1}_{r_2\times r_1}\left(\frac{2}{\pi}\mathbf{1}_{r_1}+\alpha \bI_{r_1}+\bU^t\bU\right)^{-1} \mathbf{1}_{r_1\times r_2}
= \frac{4}{\pi^2} \left(\frac{r_1}{2/\pi r_1+\alpha}\mathbf{1}_{r_2} -\frac{1}{(2/\pi r_1+\alpha)^2} \left<\mathbf{1}_{r_1},\bA \right>\mathbf{1}_{r_2}\right).
\end{align}
Therefore, using \eqref{eq:setare}, we have
\begin{align}
\frac{1}{r_2} \left<\mathbf{1},\bR/\bR_{11} \right>=\frac{2r_2}{\pi}+\alpha-\frac{(4/\pi^2)r_1r_2}{2/\pi r_1+\alpha}+\frac{(4/\pi^2)\left<\mathbf{1}_{r_1},\bA \right>r_2}{(2r_1/\pi+\alpha)^2}.
\end{align}
Therefore, we have
\begin{align}
\lim_{d\to\infty} \frac{1}{r_2} \left<\mathbf{1},\bR/\bR_{11} \right>=\alpha+\left<\mathbf{1}_{r_1},\bA \right>\frac{r_2}{r_1^2}.
\end{align}
On the other hand, since the matrix $1/\alpha \bI-2/(\alpha (\pi \alpha+2r_1))\mathbf{1}_{r_1}$ is positive semidefinite, we have
\begin{align}
\left<\mathbf{1}_{r_1},\bA\right>\leq \left<\mathbf{1}_{r_1},\bU^t \bU\right>=\|\bU\mathbf{1}_{r_1}\|^2.
\end{align}
Since columns of $\bU$ are randomly generated (e.g., using a Gaussian distribution), we have $\|\bU\|\leq 1+\sqrt{\gamma}+\mu$ with probability $1-2\exp(-\mu^2 d)$. Thus, $\|\bU\mathbf{1}\|^2\leq r(1+\sqrt{\gamma}+\mu)^2$ with probability $1-2\exp(-\mu^2 d)$. Thus, with high probability,
\begin{align}
\lim_{d\to\infty} \frac{1}{r_2} \left<\mathbf{1},\bR/\bR_{11}\right>\leq 1-\frac{2}{\pi}+(1+\sqrt{\gamma}+\mu)^2 \frac{r_2}{r_1}.
\end{align}
This along with \eqref{eq:setare0} lead to
\begin{align}
\|\bR/\bR_{11}\|\leq 1-\frac{2}{\pi}+(1+\sqrt{\gamma}+\mu)^2\frac{r_2}{r_1}
\end{align}
with probability $1-2\exp(-\mu^2 d)$. Replacing this in \eqref{eq:bad-loss2} completes the proof.

\subsection{Proof of Lemma \ref{lemma:relucontinuity}}\label{subsec:proof-lemma-relu-cont}
We consider four different cases for signs of $\left\langle \bw_1, \bx\right\rangle$, $\left\langle \bw_2, \bx\right\rangle$.
\begin{enumerate}
\item $\left\langle \bw_1, \bx\right\rangle \leq 0$, $\left\langle \bw_2, \bx\right\rangle \leq 0$: In this case,
$\phi\left(\left\langle\bw_1, \bx\right\rangle\right) = \phi\left(\left\langle\bw_2, \bx\right\rangle\right) = 0$. Hence, the lemma
statement is trivial.
\item $\left\langle \bw_1, \bx\right\rangle \geq 0$, $\left\langle \bw_2, \bx\right\rangle \geq 0$: We have
\begin{align*}
\phi\left(\left\langle\bw_1, \bx\right\rangle\right) - \phi\left(\left\langle\bw_2, \bx\right\rangle\right)  = \left\langle\bw_1, \bx\right\rangle - \left\langle\bw_2, \bx\right\rangle = \left\langle\bw_1-\bw_2, \bx\right\rangle \leq \left\|\bw_1-\bw_2\right\|_2\|\bx\|_2.
\end{align*}
\item $\left\langle \bw_1, \bx\right\rangle \geq 0$, $\left\langle \bw_2, \bx\right\rangle \leq 0$: In this case we have
\begin{align*}
\phi\left(\left\langle\bw_1, \bx\right\rangle\right) - \phi\left(\left\langle\bw_2, \bx\right\rangle\right)  = \left\langle\bw_1, \bx\right\rangle = \left\langle\bw_1-\bw_2, \bx\right\rangle + \left\langle\bw_2, \bx\right\rangle \leq \left\langle\bw_1-\bw_2, \bx\right\rangle \leq \left\|\bw_1-\bw_2\right\|_2\|\bx\|_2.
\end{align*}
\item $\left\langle \bw_1, \bx\right\rangle \leq 0$, $\left\langle \bw_2, \bx\right\rangle \geq 0$: After switching the roles of $\bw_1, \bw_2$, the proof is the same as it was in case (3).
\end{enumerate}
Therefore, the lemma statement holds in all four cases for signs of $\left\langle \bw_1, \bx\right\rangle$, $\left\langle \bw_2, \bx\right\rangle$. This completes the proof.

\subsection{Proof of Lemma \ref{lemma:spherenet}}\label{subsec:proof:lem:spherenet}
We use the result of Lemma 5.2 in \cite{vershynin2010introduction}. Let $|\mathcal U|$ be an $\eps$-net
of $H^{n-1}$, an arbitrary unit hemisphere in $n$-dimensions, where
\begin{align*}
\eps = \sqrt{2-2\cos \delta}.
\end{align*}
Using Lemma 5.2 in \cite{vershynin2010introduction},
\begin{align*}
|\mathcal U| \le \frac{1}{2}\left(1 + \frac{\sqrt{2}}{\sqrt{1-\cos \delta}}\right)^n.
\end{align*}
Now we show that $\mathcal U$ is an angular $\delta$-net of $S^{n-1}$. Let $\bv\in \reals^n$ be
an arbitrary vector in $S^{n-1}$. Note that $\mathcal U \cup\mathcal U^-$ is an $\epsilon$-net
for the unit sphere $S^{n-1}$. Hence,
there exists a vector $\bu \in \mathcal U \cup\mathcal U^-$, such that
\begin{align}
\|\bu - \bv\|_2^2 \leq \eps^2 = 2-2\cos\delta.
\end{align}
Thus,
\begin{align*}
\|\bu\|_2^2 + \|\bv\|_2^2 - 2\|\bu\|\|\bv\|\cos \theta_{\bu,\bv} = 2-2\cos\theta_{\bu,\bv} \leq 2 - 2\cos \delta.
\end{align*}
Therefore,
\begin{align*}
\cos \theta_{\bu,\bv} \ge \cos\delta \Rightarrow \theta_{\bu,\bv}\le \delta.
\end{align*}
Hence, for every vector $\bv \in S^{n-1}$, there exists $\bu \in \mathcal U \cup\mathcal U^-$, such that
\begin{align*}
\theta_{\bu,\bv}\le \delta.
\end{align*}
This completes the proof.

\subsection{Proof of Theorem \ref{thm:minimaxrisk}}\label{subsec:proof:thm:minimax}
Let $f^*(\bx) = h(\bx; \bw_1^*, \bw_2^*, \dots, \bw_k^*)$, for a set of weights $\bw_i^* \in \mathcal W$,
be an arbitrary member of $\mathcal F$.
Since $\mathcal U$ is an angular $\delta$-net of $\mathcal W$, for $i = 1,2,\dots,k$,
we can take $\tilde\bu_i\in \mathcal U\cup\mathcal U^-$ such that $\theta_{\tilde\bu_i, \bw_i^*}\le \delta$.
For $i=1,2,\dots, k$, take $\tilde\bw_i \in \mathcal W_\mathcal U$ as
\begin{align*}
\tilde\bw_i = \frac{\|\bw^*_i\|}{\|\tilde\bu_i\|}\tilde\bu_i.
\end{align*}
Note that we have
\begin{align}
\label{eq:normw*-wtilde}
\|\bw_i^* - \tilde\bw_i\|_2^2 &= \|\bw_i^*\|_2^2 + \|\tilde\bw_i\|_2^2 - 2\|\tilde\bw_i\|_2 \|\bw_i^*\|_2\cos\theta_{\tilde\bu_i, \bw_i^*}\nonumber\\
&= 2\|\bw_i^*\|_2^2(1-\cos\theta_{\tilde\bu_i, \bw_i^*}) \le 2\|\bw_i^*\|_2^2(1-\cos\delta).
\end{align}
Taking $\tilde f(\bx) = h(\bx; \tilde\bw_1, \tilde\bw_2, \dots, \tilde\bw_k)\in \mathcal F_\mathcal L$, we have
\begin{align*}
\min_{\hat f \in \mathcal F_\mathcal L} \EE\left|f(\bx) - \hat f(\bx)\right| \leq \EE|f(\bx) - \tilde f(\bx)|
&\le \EE |h(\bx;\bw_1^*,\bw_2^*,\dots, \bw_k^*)-h(\bx;\tilde\bw_1,\tilde\bw_2,\dots, \tilde\bw_k^*)|\\
&\le \EE\left|\sum_{i=1}^k \phi\left(\left\langle\bw_i^*,\bx\right\rangle\right)-\sum_{i=1}^k \phi\left(\left\langle\tilde\bw_i,\bx\right\rangle\right)\right|\\
&\le \EE \sum_{i=1}^k\left|\phi\left(\left\langle\bw_i^*,\bx\right\rangle\right) - \phi\left(\left\langle\tilde\bw_i,\bx\right\rangle\right)\right|.
\end{align*}
Using Lemma \ref{lemma:relucontinuity}, we get
\begin{align*}
\min_{\hat f \in \mathcal F_\mathcal L} \EE\left|f(\bx) - \hat f(\bx)\right| \leq \left(\sum_{i=1}^k\|\bw_i^*-\tilde\bw_i\|_2\right) \EE\|\bx\|_2 = \sqrt{d}\sum_{i=1}^k\|\bw_i^*-\tilde\bw_i\|_2
\end{align*}
Hence, by \eqref{eq:normw*-wtilde}
\begin{align}
\min_{\hat f \in \mathcal F_\mathcal L}\EE\left|f(\bx) - \hat f(\bx)\right| \leq \sqrt{2d(1-\cos \delta)}\sum_{i=1}^k\|\bw_i^*\|_2 \le
kM\sqrt{2d(1-\cos\delta)}.
\end{align}
Thus,
\begin{align*}
\mathcal R\left(\mathcal F_\mathcal L, \mathcal F\right) = \max_{f\in \mathcal F}\min_{\hat f \in \mathcal F_\mathcal V} \EE\left|f(\bx) - \hat f(\bx)\right| \le kM\sqrt{2d(1-\cos\delta)}.
\end{align*}

\end{document}